\documentclass{article}

\usepackage[utf8]{inputenc}
\usepackage[T1]{fontenc}

  \usepackage{mathpazo}
  \usepackage[letterpaper,margin=1in]{geometry}
\usepackage[
  backend=biber,
  style=alphabetic,
  natbib=true,
  maxbibnames=99,
   maxalphanames=3,
  minalphanames=3
]{biblatex}
  \date{}

\usepackage{microtype}
\usepackage{graphicx}
\usepackage{subcaption}
\usepackage{booktabs}
\usepackage{tabularx}
\usepackage{float}
\usepackage{needspace}
\usepackage{amsmath}
\usepackage{amssymb}
\usepackage{amsfonts}
\usepackage{mathtools}
\usepackage{amsthm}
\usepackage{comment}
\usepackage{nicefrac}
\usepackage{xcolor}
\usepackage{url}
\usepackage{hyperref}
  \hypersetup{colorlinks=true,allcolors=blue}
\usepackage[capitalize,noabbrev]{cleveref}

\theoremstyle{plain}
\newtheorem{theorem}{Theorem}[section]
\newtheorem{informaltheorem}[theorem]{Informal Theorem}

\theoremstyle{definition}

\theoremstyle{remark}

\usepackage[disable,textsize=tiny]{todonotes}

\title{High-Value Training Data from Noisy, Diverse Sources}

\author{%
  Hedi Driss$^{1,*}$, 
  Evan Frick$^{2,*}$, 
  Anastasios N. Angelopoulos$^2$,
  Cho-Jui Hsieh$^{2,3}$, \\
  Yejin Huh$^{1}$,
  Tijana Zrnic$^{2,4}$,
  Andrew Ilyas$^{1,5}$ \\[1em]
  $^1$Hiddenweights
  \quad 
  $^2$Arena 
  \quad 
  $^3$UCLA 
  \quad 
  $^4$Stanford 
  \quad 
  $^5$CMU\\
  {\normalsize $^*$Equal contribution
  }
}

\begin{document}

\begingroup
\setlength{\parskip}{3pt plus 1pt minus 0.5pt}
\setlength{\topsep}{3pt plus 1pt minus 1pt}
\setlength{\partopsep}{0pt}
\setlength{\itemsep}{1pt plus 0.5pt minus 0.5pt}
\setlength{\parsep}{0pt}
\captionsetup{font=small,skip=6pt}
\captionsetup[subfigure]{font=small,skip=4pt}
\setlength{\textfloatsep}{8pt plus 2pt minus 2pt}
\setlength{\floatsep}{8pt plus 2pt minus 2pt}
\setlength{\intextsep}{8pt plus 2pt minus 2pt}
\setcounter{topnumber}{3}
\setcounter{bottomnumber}{3}
\setcounter{totalnumber}{5}
\renewcommand{\topfraction}{0.95}
\renewcommand{\bottomfraction}{0.9}
\renewcommand{\textfraction}{0.05}
\renewcommand{\floatpagefraction}{0.75}
\makeatletter
\setlength{\@fptop}{0pt}
\setlength{\@fpsep}{8pt plus 2pt minus 2pt}
\makeatother
\raggedbottom

\maketitle

\begin{abstract}
Training state-of-the-art models increasingly requires learning from large, heterogeneous datasets whose examples vary in quality, domain relevance, and source of supervision. Dataweighting---assigning different weights to training examples---offers a general way to control how such data influences training, but learning effective weights at scale remains challenging. We develop a scalable approach to learning dataweights: given a small, high-quality target dataset, we use metagradients through the training process to efficiently optimize training-example weights directly for performance on the target domain. To scale beyond a fixed dataset and model, we learn parametric weighting rules that generalize to new and streaming examples and transfer from smaller models to substantially larger ones.
We study reward modeling as a prototypical application, where preference data
may be noisy, adversarial, synthetic, or mismatched with the target domain.
Learned dataweights improve held-out performance, automatically adapt reward
models to specialized domains, and allow mixtures of human and synthetic
supervision to outperform either source~alone.
\end{abstract}

\section{Introduction}

In modern machine learning, we increasingly want to train reliable, specialized models using heterogeneous, general-purpose data. For example, large language models are trained on massive corpora scraped from the internet; vision models are trained on heterogeneous collections of web images. Domain-specific models are often adapted from such general-purpose models using much smaller, specialized datasets. This process poses several challenges, including low-quality or missing data, distribution shift, and mismatches between the training objective and the downstream evaluation criterion.

Standard approaches to these challenges are often bespoke. For example, we might design heuristic filters to remove low-quality data \citep[e.g.,][]{penedo2023refinedweb}; synthesize labels where they are missing \citep[e.g.,][]{xie2020self}; or explore data mixing strategies between general and target-specific data \citep[e.g.,][]{xie2023doremi}. While these techniques can be effective, they require substantial problem-specific design and often rely on human judgments about what constitutes ``good'' or ``bad'' data---judgments that may themselves be unreliable or unavailable in many domains. As the need for customizable models and the amount of available general-purpose data grow, it is natural to seek a more scalable approach.

An attractive general paradigm is to assign a weight to each training example \citep{huang2006correcting,ren2018learning,shu2019meta}. We call this practice {\em dataweighting}. In theory, dataweighting can reduce the influence of outliers and low-quality examples and reshape the effective training distribution to better align with the downstream task. In practice, however, choosing the right weights is difficult, particularly when little is known a priori about the relationship between the training data and the target domain. In this paper, we therefore ask:
\begin{center}
\emph{Is there a \underline{scalable} way to \underline{learn} dataweights for training on diverse, heterogeneous data?}
\end{center}

\paragraph{Contributions.} In this work, we revisit {\em dataweighting} as a unified,
scalable way of reducing both the human effort and subjective
decision-making involved in training specialized models. 
Our central premise is that a training sample's utility should be
determined by its impact on the performance
of the resulting model,
rather than by external heuristics. Instead of
predefining quality criteria, we infer data usefulness directly from
the learning process itself, and reweight the data's impact appropriately.

We provide 
algorithms and principles for instantiating \textcolor{black}{dataweighting} \emph{at scale}. We leverage tools from data attribution
(and specifically metadifferentiation \citep{ilyas2025magic,engstrom2025optimizing})
to quantify how individual training examples influence a model's behavior.
In this way, data reweighting becomes an adaptive procedure that
is directly aligned with the target objective.
We show how to go beyond a fixed training dataset and training architecture by implementing a {parametric} {\em dataweighting model} and showing that dataweights learned with a smaller model can transfer to substantially larger ones. 

We identify reward modeling as a prototypical application of the challenges we study. Preference data is heterogeneous and may contain spam, bot-generated responses, inconsistent judgments, or subjective preferences. Large pretrained language models offer an attractive source of auxiliary supervision, but their labels can introduce additional uncertainty and bias. Moreover, the reward models' training objective (Bradley-Terry loss) is often not aligned with
the downstream uses that we care about, such as providing accurate model rankings.

Our experiments with human and synthetic preference data support four main takeaways,
validated \emph{at scale} with models whose size is on the order of billions of parameters:
\begin{enumerate}
\item \textbf{Improved performance with heterogeneous, noisy data.}
Learned dataweights improve model performance by upweighting helpful data points and downweighting harmful ones, such as outliers and adversarial examples.
\item \textbf{Domain customization.}
When the target dataset comes from a specialized domain such as coding or STEM question answering, dataweights automatically identify relevant examples within diverse training data.
\item \textbf{Effective integration of synthetic supervision.}
When data consists of a mix of human and synthetic labels, dataweights can automatically adapt to the quality of the synthetic labels, upweighting the helpful ones and downweighting the harmful ones. This results in improved performance over both (a) using real human data alone, and (b) using \emph{all} human and synthetic data while treating them equally.
\item \textbf{Alignment with downstream objectives.}
Dataweighting allows us to close the gap between the training loss 
and the downstream task we care about; in the reward modeling case, we can optimize dataweights for model ranking accuracy while retaining the standard Bradley--Terry training objective.
\end{enumerate}

Finally, in Section \ref{sec:theory} we provide a few stylized examples
illustrating these same benefits of dataweighting in theory, with the goal of 
providing further intuition about how dataweights work.

Our work builds on validation-guided example reweighting~\citep{ren2018learning,shu2019meta,calian2025datarater}
and scalable metagradient computation~\citep{engstrom2025optimizing,ilyas2025magic}.
We investigate how these tools can improve training through continuous
reweighting over long training segments, transfer of learned weighting functions
to new data and larger models, and optimization of downstream ranking objectives;
Section~\ref{sec:related} gives the detailed comparisons.

\section{Learning dataweights at scale}
\label{sec:dataweighting}

We describe the general problem of model training from noisy, heterogeneous data using dataweights. We then describe how we instantiate dataweighting for this 
problem at scale.

\subsection{Problem statement}

We are given a training dataset of size $n$ and a
target dataset of size $n'$.
The training dataset is a large, heterogeneous dataset composed of many
potentially noisy sources.
In contrast, the target dataset is a small, highly-curated gold-standard dataset
that represents the target domain.
Our goal is to train a model $\theta$, using all the available samples, that performs
well on the target domain.
Write the training examples as $\{z_i\}_{i=1}^n$ and the target
examples as $\{z'_j\}_{j=1}^{n'}$, with each example in a space $\mathcal{Z}$.
Let $\ell(z;\theta)$ denote the per-example training loss of the model with
parameters $\theta$.

Dataweighting modifies the training objective by assigning weights to each training data point. For any vector of nonnegative weights $w \in \mathbb{R}_{\ge 0}^n$, 
define the reweighted loss function
\begin{equation}
    \label{eq:train-loss}
    L^{w}(\theta) = \sum_{i=1}^{n} w_i \ell(z_i; \theta).
\end{equation}
Also define $\mathcal{A}(w,\theta_0,k)$ to be any iterative  model training
procedure that optimizes the weighted loss $L^w(\theta)$ starting from
initialization $\theta_0$ for $k$ steps; note that here, $\mathcal{A}$ captures
the model architecture, optimizer, hyperparameters, and in general everything
about training other than the weight vector $w$, initialization $\theta_0$, and
step count $k$. Let $T$ be the total number of steps that can be taken by $\mathcal{A}$
(i.e., the total length of the algorithm's learning rate schedule).

Our ultimate objective is to train a model with good performance on the target domain. Let $L'(\theta)$ denote the target objective evaluated
on the model with parameters $\theta$. For example, this could be the average loss on the target set:
\[L'(\theta)=\frac{1}{n'}\sum_{j=1}^{n'}\ell(z'_j;\theta).\]
In general, however, the target objective need not use the same loss as
training; later we explore more complex target objectives.

Our methods leverage the target dataset to learn dataweights that place more mass on relevant training data.
Formally, we learn a choice of weights $w$ such that the  model obtained by optimizing $L^{w}$ performs well on the target loss $L'(\theta)$.
In essence, our dataweights seek to solve the following problem:
\begin{equation}
    \label{eq:dataweights}
    \underset{w\in\mathbb{R}_{\ge 0}^n}{\mathrm{minimize}} \quad L'(\mathcal{A}(w,\theta_0,T)),
\end{equation}
given an initialization $\theta_0$. 
This objective encourages weights that improve target performance
by changing how strongly each training example influences the model. However, solving \eqref{eq:dataweights} outright is computationally infeasible. We discuss strategies that \emph{approximate} solving \eqref{eq:dataweights}, while remaining feasible even for models with billions of parameters.

\subsection{Scalable dataweighting}

We now describe our procedure for approximating \eqref{eq:dataweights} at the scale of 
modern models and datasets. 
A core primitive in our method is computation of the {\em metagradient}
$\nabla_w L'(\mathcal{A}(w,\theta_0,k))$,
the derivative of the target objective with respect 
to the data weights (i.e., {\em through} the algorithm $\mathcal{A}$).
Despite the complexity of computing this quantity, 
a line of work in data attribution 
\citep{koh2017understanding, ghorbani2019data, ilyas2022datamodels, grosse2023studying, park2023trak, ilyas2024data}
offers efficient ways to compute approximate metagradients,
and recent work shows how to compute these gradients {\em exactly} 
for moderate-sized models and datasets \citep{engstrom2025optimizing,ilyas2025magic}.
Using metagradient computation as a primitive, we describe a modular 
framework for learning optimal dataweights.

\paragraph{Backbone: scalable metagradient descent.}
For algorithm $\mathcal{A}$, set of dataweights $w$, 
initial parameters $\theta$, and number of training steps $k < T$,
we take as given the ability to compute a {\em metagradient} of the form 
\[
    g(w, \theta, k) := \nabla_w L'\!\left(\mathcal{A}(w,\theta,k)\right),
\]
which captures the effect of infinitesimally up- or down-weighting each point 
in the training set on the final target objective after training.
Note that we will not discuss the efficient computation of metagradients here,
and instead refer to Section 3.3 of \citet{ilyas2025magic} for a thorough treatment.

With this primitive in hand, our algorithm at a high level will be that of 
{\em metagradient descent} (MGD), as introduced by \citet{engstrom2025optimizing}.
In MGD, we initialize $\smash{w^{(0)} = \mathbf{1}_n}$ and $\theta_0$, and then repeatedly apply
\begin{equation}
    \label{eq:mgd-basic}    
    w^{(t+1)} \gets \left[w^{(t)} - \eta \cdot g\!\left(w^{(t)}, \theta_0, T\right)\right]_+, 
\end{equation}
where $T$ is the number of steps used to train the model and $[\cdot]_+$ represents non-negative 
clipping. 
For clarity, we refer to steps taken by the algorithm $\mathcal{A}$
to update model parameters $\theta$ as {\em inner steps}, 
and metagradient steps taken on dataweights $w$ as {\em outer steps}.
While metagradient descent 
is the basis of our method, we introduce several extensions that 
allow us to apply it to bigger models trained on large, heterogenous, potentially 
streaming training data sets.

\paragraph{Parameterizing dataweights.}
We first extend \eqref{eq:mgd-basic} to large (and potentially
streaming) training datasets by introducing \emph{parametric} dataweights. Formally, parametric dataweights are defined by a function
$\smash{h_\psi: \mathcal{Z} \to \mathbb{R}_{\geq 0}}$, parameterized by $\smash{\psi \in \mathbb{R}^p}$. The corresponding dataweights are then obtained as $\smash{w_i^\psi = h_\psi(z_i)}$. We compute metagradients 
with respect to the underlying parameters $\psi$ via the chain rule, i.e.,
\begin{equation}
    \label{eq:parametric-chain-rule}
\nabla_\psi L'\!\left(\mathcal{A}\!\left(w^\psi, \theta, k\right)\right) = 
\sum_{i=1}^n \left[\nabla_w L'\!\left(\mathcal{A}\!\left(w^\psi, \theta, k\right)\right)\right]_i \cdot \nabla_\psi h_\psi(z_i).
\end{equation}
We recover basic MGD via the {\em identity} parameterization,
where we set $p = n$, such that
\begin{equation}
    \label{eq:np-w}
    \psi \in \mathbb{R}^n \qquad \text{and} \qquad
   w_i^\psi = h_\psi(z_i) = \psi_i.
\end{equation}
There are two main drawbacks to this trivial parameterization.
First, there is one learnable parameter 
per training point (the scalar weight $\psi_i$), 
making the dataweight optimization extremely high-dimensional and overparameterized. 
Second, it is unclear how to handle training data that
arrives only \emph{after} we have trained a model; for example, if we want to retrain or fine-tune the model on newly collected data.

A solution to both of these challenges is to flexibly 
parameterize the dataweights $w$ as a function of features of the
corresponding training point. 
As a concrete example---which we call the {\em linear} parameterization---we let
$\phi: \mathcal{Z} \to \mathbb{R}^{d}$
be a fixed feature map (e.g., a frozen embedding of the example).
We then let
\begin{equation}
\label{eq:parametric-w}
    \psi \in \mathbb{R}^{d+1} \qquad \text{and} \qquad 
    w_i^{\psi} = h_\psi(z_i) = \log(1 + \exp(\phi(z_i)^\top \psi + \psi_0)).
\end{equation}
By choosing $d\ll n$ and optimizing $\psi \in \mathbb{R}^{d+1}$, rather than the per-example weights 
$\psi \in \mathbb{R}^n$, we reduce the 
dimensionality of the optimization problem and also enable extrapolation to 
unseen training points. That is, we can incorporate new data that was not seen 
during the dataweight optimization process by simply assigning it a weight 
using the map $h_\psi$ \eqref{eq:parametric-w}. Thus, even if a model is repeatedly retrained over time on freshly collected data, one can apply dataweighing by learning a parameterized dataweighting map only \emph{once}.

\paragraph{Segmenting the training trajectory.}
The next change we make is to differentiate with respect to \emph{segments} of the
training trajectory rather than over entire training runs. Specifically, rather than 
differentiating through $\mathcal{A}(w, \theta_0, T)$ where $T$ is the total number of 
training steps, we choose a relatively small number of {\em segments} $K$ such that $T \mod K = 0$.
In lieu of \eqref{eq:mgd-basic}, we initialize both $\psi^{(0)}$ {\em and} $\theta^{(0)} = \theta_0$, 
then for $t=1,2,\dots$ we:
\begin{itemize}
    \item[(a)] compute $w^{(t-1)} \gets [h_{\psi^{(t-1)}}(z_i)]_{i=1}^n$, the dataweights 
    induced by $\psi^{(t-1)}$;
    \item[(b)] update the model parameters $\theta^{(t)} \gets \mathcal{A}(w^{(t-1)}, \theta^{(t-1)}, \frac{T}{K})$, resetting $\theta^{(t)} \gets \theta_0$ if $t \mod K = 0$;
    \item[(c)] compute $\nabla_w L'(\mathcal{A}(w^{(t-1)}, \theta^{(t-1)}, \frac{T}{K}))$ and update the underlying $\psi^{(t)}$ using gradient \eqref{eq:parametric-chain-rule}:
    \vspace*{-.2em}
    \begin{equation*}
    \psi^{(t)} \gets \psi^{(t-1)} - \eta \cdot \nabla_\psi L'\left(\mathcal A\left(w^{(t-1)}, \theta^{(t-1)}, \nicefrac{T}{K}\right)\right).
\end{equation*}\vspace*{-1.3em}
\end{itemize}
Every $K$ steps is a full run of $\mathcal{A}$;
we call this a {\em meta-epoch} and use $\Omega$ ($=\!\lfloor \nicefrac{t}{K}\rfloor$) to track the count. 
When $K = 1$ (and we use the identity parameterization \eqref{eq:np-w}), we recover the ordinary metagradient 
descent algorithm. On the other extreme, when $K = T$ (or within a small constant factor of $T$),
we recover a variation on prior approaches to dataweighting, which approximate the value of a given 
datapoint by differentiating through a few steps of the training procedure $\mathcal{A}$ (e.g., 
\citep{ren2018learning,calian2025datarater}). In our experiments, we use $K = \{1, 10, 20\}$
depending on the exact setting. We find that increasing $K$ to be on the order of $T$ significantly 
harms the method's performance, which may partially explain some prior negative results.

Segmentation increases the frequency with which we update the parameter $\psi$
(and therefore the dataweights $w_i$).
In particular, updating the dataweights $t$ times requires running the training 
procedure for $\smash{\frac{t \cdot T}{K}}$ total optimization steps, and so 
increasing $K$ reduces the cost of optimization linearly.

\paragraph{Transferring dataweights across scales.}
We further improve the scalability of the algorithm by leveraging 
{\em cross-algorithm transfer} of learned dataweights $w$. 
The general idea here is that given an algorithm $\mathcal{A}$ that may 
be costly to run (e.g., full-parameter fine-tuning of a 70B-parameter 
language model), we can define a {\em proxy algorithm} $\mathcal{A}'$ 
(e.g., LoRA fine-tuning of a 7B-parameter language model with appropriately scaled hyperparameters) and perform dataweight optimization using the proxy. Thus, we first optimize dataweights $w_i$ using $\mathcal{A}'$, and then retrain a model using  $\mathcal{A}$ by plugging the learned weights into objective \eqref{eq:train-loss}.

\paragraph{Optimizing complex downstream objectives.}
Finally, we explore the design of the target objective $L'(\theta)$.
By default, we defined $L'(\theta)$ as the average loss 
of a model with parameters $\theta$ on a high-quality target set.
Notice, however, that the metagradient 
$\nabla_w L'(\mathcal{A}'(w, \theta, k))$ depends on $L'(\theta)$ only through 
a final ``target gradient'' $\nabla_\theta L'(\theta)$---meaning that 
we can use {\em any} differentiable function of model parameters as 
the target objective. 
Furthermore, because the metagradient computation itself is computationally
expensive relative to evaluation of a trained model, we can encode complex 
measurements of desired behavior into $L'$ (e.g., $L'$ might take 
several minutes to evaluate and thus be totally impractical to train on)
at essentially no wall-clock overhead to our algorithm, so long as 
evaluating and differentiating $L'$ costs less than $\nicefrac{T}{K}$ training steps.

\section{Applying scalable dataweighting to reward models}
\label{sec:results}
\begingroup
\captionsetup{font=small}
\captionsetup[subfigure]{font=small}
\setlength{\textfloatsep}{8pt plus 2pt minus 2pt}
\setlength{\floatsep}{8pt plus 2pt minus 2pt}
\setlength{\intextsep}{8pt plus 2pt minus 2pt}
We identify {\em reward modeling} as an ideal testbed for 
training specialized models from heterogeneous data. In reward modeling, our input is {\em preference data}:
a collection of prompts, each paired with two candidate responses and a label
indicating which response is preferred. The goal is to learn a model that takes
a prompt and a response and assigns a numerical score, allowing us to compare
new responses by predicting which one would be preferred. We train the model
with the Bradley-Terry objective~\citep{bradley1952rank,christiano2017deep},
which encourages higher scores for preferred responses.
Formally, each training example is
$z_i = (\text{prompt}_i, \text{pref\_resp}_i, \text{unpref\_resp}_i)$ 
and the corresponding training loss for a reward model $r_\theta$
is 
$$\ell(z_i;\theta) := -\log \sigma\left(r_\theta(\text{prompt}_i,\text{pref\_resp}_i) - r_\theta(\text{prompt}_i,\text{unpref\_resp}_i)\right).$$
Reward modeling is an excellent candidate for dataweighting:
preference data is noisy, often based on LLM judges, and comes from diverse sources.
We demonstrate the efficacy of dataweighting on real human preference data in a variety
of settings. We begin by giving an overview of the training and target datasets
we use, as well as our reward model training procedure.

\subsection{Experimental setup}
We demonstrate the efficacy of dataweighting 
for reward modeling across a variety of training datasets, target objectives,
and model architectures. We also compare dataweighting to existing algorithms 
for data-centric adaptation from the literature. 
We provide a high-level summary of these settings here, with a more complete 
description available in Appendix~\ref{app:experimental-setup}.

\paragraph{Data.} We use preference data released by Arena in our analysis. 
Specifically, for training we make use of:
\vspace*{-.5em}
\begin{itemize}
    \setlength{\itemsep}{2pt}
    \setlength{\parsep}{0pt}
    \setlength{\parskip}{0pt}
    \item \texttt{Human-155k}~\citep{chiang24chatbot}: $155{,}577$
    preference data points collected from Arena users;
    \item \texttt{Bot-10k}: $10{,}000$ Arena-flagged bot and adversarial examples;
    \item \texttt{Gemini-155k}: Synthetic Gemini 3 Flash preference labels on the prompt-response pairs in \texttt{Human-155k}. 
\end{itemize}
\vspace*{-.5em}
Datasets such as \texttt{Human-45k} refer to 
random subsamples of the corresponding datasets; mixtures such as \texttt{Human-45k + Gemini-110k} 
refer to combined subsamples with disjoint rows.
As target datasets, we use:
\vspace*{-.5em}
\begin{itemize}
    \setlength{\itemsep}{1pt}
    \setlength{\parsep}{0pt}
    \setlength{\parskip}{0pt}
    \item \texttt{Expert}: a curated subset of Arena preferences for
    sophisticated, domain-specific prompts, testing gains on a niche part of the
    training distribution. 
    \item \texttt{STEM}: math, science, and coding preferences from
    HelpSteer3 and Tulu2.5~\citep{wang2025helpsteer3preferenceopenhumanannotatedpreference,ivison2024unpacking},
    with annotations from qualified human rubric labelers or ground-truth
    verifiers.
\end{itemize}
Each domain is split into a target set used for optimizing dataweights and an entirely held-out test set:
$3000$/$547$ examples for \texttt{Expert} and $6796$/$6796$ for \texttt{STEM}.

\paragraph{Reward model training.} We fine-tune \texttt{Qwen2.5-\allowbreak0.5B-\allowbreak Instruct} unless otherwise specified.
{Except where otherwise stated,
for fair comparisons {\em all} methods use the same fixed-size
\emph{checkpoint selection set} of $100$ examples for early stopping,
disjoint from both the target and test sets.}
{Comparisons between dataweighted and unweighted models use matched random seeds.}
Appendix~\ref{app:experimental-setup} details dataset construction, sampling,
and reward-head implementation.
Appendix~\ref{app:run-hyperparameters} lists the training hyperparameters,
dataweight optimization settings, and parametric weighting model configuration.

\subsection{Establishing the efficacy of dataweighting}
\label{sec:selection-baselines}
We start by establishing the efficacy of our proposed method for dataweighting
relative to existing ideas in the literature on data-centric model adaptation. 
Specifically, we compare our method
to:
\begin{itemize}
    \item {\bf Attribution-based data selection methods} (approximate influence
    \citep{koh2017understanding}, TracIn \citep{pruthi2020estimating}, and TRAK
    \citep{park2023trak}), which assign a score to each training example based 
    on its estimated contribution to the target objective, and then keep only the 
    highest-scoring examples;
    \item {\bf Metagradient-based methods over a short horizon} (as in \cite{ren2018learning}
    and others), where we attempt to optimize dataweights $w$ by differentiating through 
    single inner optimization steps---in other words, our method but with $K = T$.
    \item {\bf Fine-tuning}, where we leverage the high-quality target data
    by training directly on it, instead of using it to optimize 
    dataweights for the training data. Specifically, we study a method which first 
    {\em pre-trains} on the $n$ datapoints from the heterogenous training set, 
    and then {\em fine-tunes} on the $n'$ target samples.
\end{itemize}
We also evaluate the unweighted baseline, which sets all weights $w_i$ equal to one.

\paragraph{Dataweighting outperforms attribution-based selection.} In 
Figure \ref{fig:selection-baselines}, we compare our dataweighting approach with the aforementioned attribution baselines using 
\texttt{Human-155k} as the training set and the validation split of 
\texttt{STEM} as the target set, evaluating on the corresponding (unseen) test
split. 
We use a small fixed pool of held-out examples to determine how many samples to 
remove for the baseline methods, and use the same fixed pool to decide when to 
early stop our dataweighting method. 

\begin{figure*}[!ht]
    \centering
    \includegraphics[width=.9\linewidth]{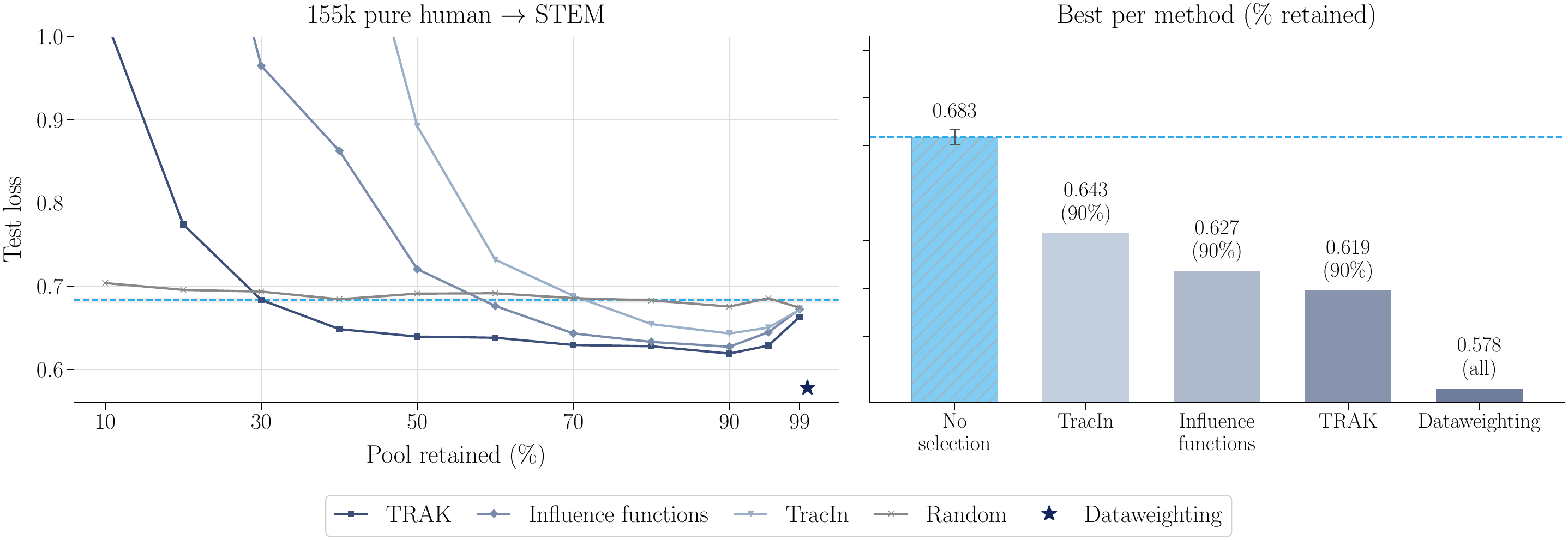}
    \caption{{\em Dataweighting outperforms attribution-based selection.}
    Held-out \texttt{STEM} test loss after training on \texttt{Human-155k}.
    \emph{Left:} Baselines across different selection fractions, with the unweighted
    baseline and a gray band showing $\pm 1$ standard deviation over three
    seeds. Dataweighting reweights the full pool (keeping $100\%$ of data points).
    \emph{Right:} Each method at its best selection fraction (in parentheses); the leftmost
    bar is the unweighted baseline. Dataweighting reaches $0.578$,
    compared with $0.619$ for TRAK, the strongest selection baseline.
    See Appendix~\ref{app:selection-pools} for other training datasets.}
    \label{fig:selection-baselines}
\end{figure*}

Overall, we find that dataweighting achieves a lower loss 
on the unseen \texttt{STEM} data than any of the baselines, and reduces test 
loss compared to unweighted training on \texttt{Human-155k} by 15\%. 
Between the data selection methods surveyed, TRAK \citep{park2023trak}
appears to be both more effective than TracIn and influence functions 
in terms of the final test loss, as well as more robust to the choice 
of selection fraction.
More qualitatively, at low selection fractions
we find that selected preferences become skewed toward longer responses,
particularly for TracIn and influence functions. Such a skew can encourage the
reward model to use response length as a shortcut, assigning low probability to
preferences favoring longer responses and incurring large losses on those pairs.
This offers a plausible explanation for the sharp loss increases under aggressive
selection (see Appendix~\ref{app:selection-length-audit} for a more
thorough investigation).

\begin{figure}[!ht]
    \centering
    \begin{minipage}[t]{0.50\linewidth}
        \vspace{0pt}
        \centering
        \includegraphics[width=\linewidth]{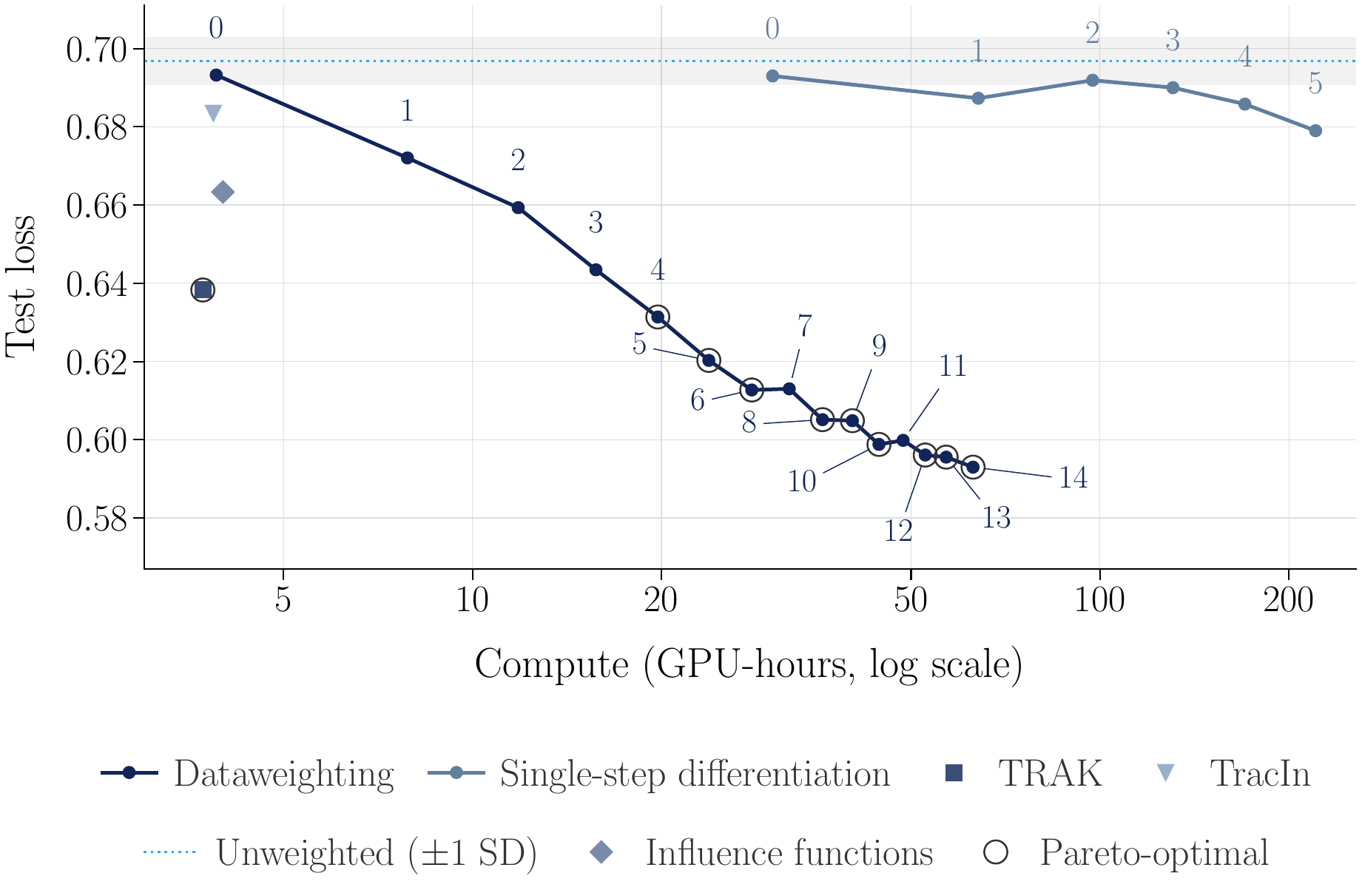}
    \end{minipage}\hfill
    \begin{minipage}[t]{0.46\linewidth}
        \vspace{0pt}
        \vspace{6pt}
        \captionsetup{justification=raggedright,singlelinecheck=false,skip=0pt}
        \caption{
        {\em Dataweighting improves the cost--performance tradeoff.}
        Held-out \texttt{STEM} loss versus compute for
        \texttt{Human-30k}, comparing dataweighting,
        single-step differentiation, and attribution-based selection.
        Point labels indicate meta-epochs $\Omega$ (i.e., complete 
        runs of the algorithm $\mathcal{A}$); black rings mark Pareto-optimal
        observations. The gray band shows the unweighted baseline mean and $\pm 1$ standard
        deviation across three seeds. Additional results appear in
        Appendix~\ref{app:selection-cost-performance}.
        \label{fig:selection-cost-benefit}}
    \end{minipage}
    \vspace{-4pt}
\end{figure}

\paragraph{Dataweighting achieves Pareto-optimal cost-performance tradeoffs.}
In our method, the number of outer steps provides a cost-performance tradeoff: later iterations can improve test loss at higher
cumulative compute. Figure~\ref{fig:selection-cost-benefit} compares
the trajectory of dataweighting with single-step differentiation and
attribution-based selection in terms of the achieved test loss and compute. We
observe that dataweighting is Pareto optimal when there are at least four
meta-epochs.
Appendix~\ref{app:selection-baselines} provides
additional details and results.

\begin{figure}[!ht]
    \centering
    \includegraphics[width=.9\linewidth]{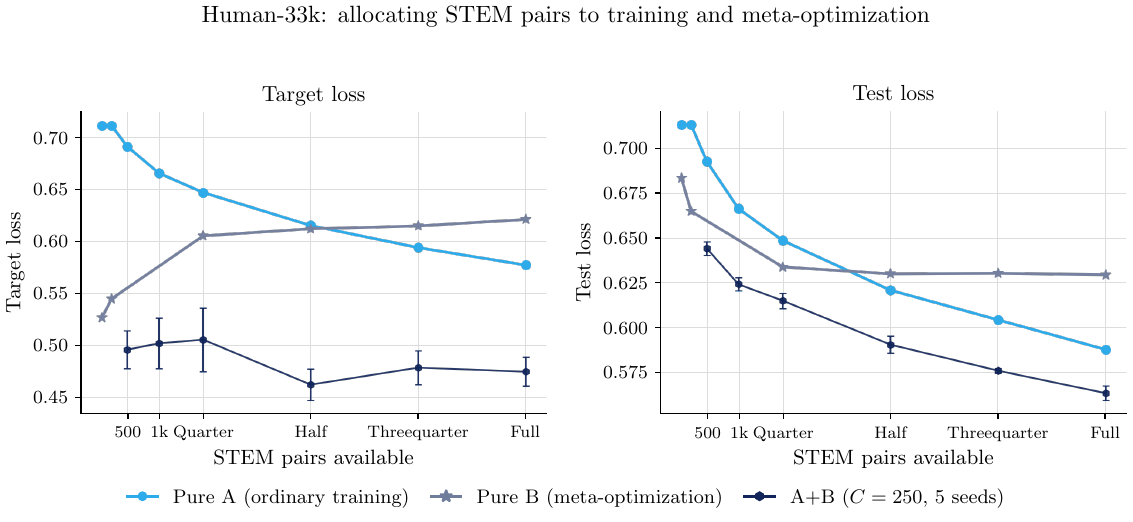}
    \captionsetup{skip=4pt}
    \caption{{\em A small \texttt{STEM} target set complements direct training on \texttt{STEM} data.}
    Target loss (left) and held-out test loss (right)
    with a \texttt{Human-33k} training base. Target-domain (\texttt{STEM}) pairs are allocated
    entirely to ordinary training (Pure A), entirely to the meta-optimization target
    (Pure B), or \textcolor{black}{split, with $250$ target pairs and the rest used for training} (A+B).
    Full denotes $6{,}796$ pairs; Quarter, Half, and Threequarter are fractions of
    \textcolor{black}{the full target set}.}
    \label{fig:expert-split}
    \vspace{-4pt}
\end{figure}

\paragraph{Dataweighting outperforms fine-tuning, especially when high-quality data is scarce.}
A natural question is whether reserving target-domain data for dataweighting
is more beneficial than using it directly for training.
Given a budget of $B$ \texttt{STEM} preference pairs, we allocate $C$ to the
dataweighting target and $B-C$ to training. We compare $C=0$ (no dataweighting),
$C=B$ (all pairs reserved for dataweighting), and $C=250$ (split).
The split achieves the lowest held-out test loss at every budget
(Figure~\ref{fig:expert-split}), outperforming ordinary training even though it uses
fewer target-domain pairs for training.

\subsection{Revisiting our motivating challenges}
\label{sec:motivating_challenges}
Having established the efficacy of our dataweighting method relative to natural baselines, 
we now turn to the motivating applications of dataweights.
Specifically, we show how dataweights allows us to circumvent many of the named 
challenges of building reliable models from heterogenous, noisy, and partially synthetic data.
Figure \ref{fig:dataweighting} summarizes our results.

\begin{figure*}[!ht]
    \centering
    \includegraphics[width=\linewidth]{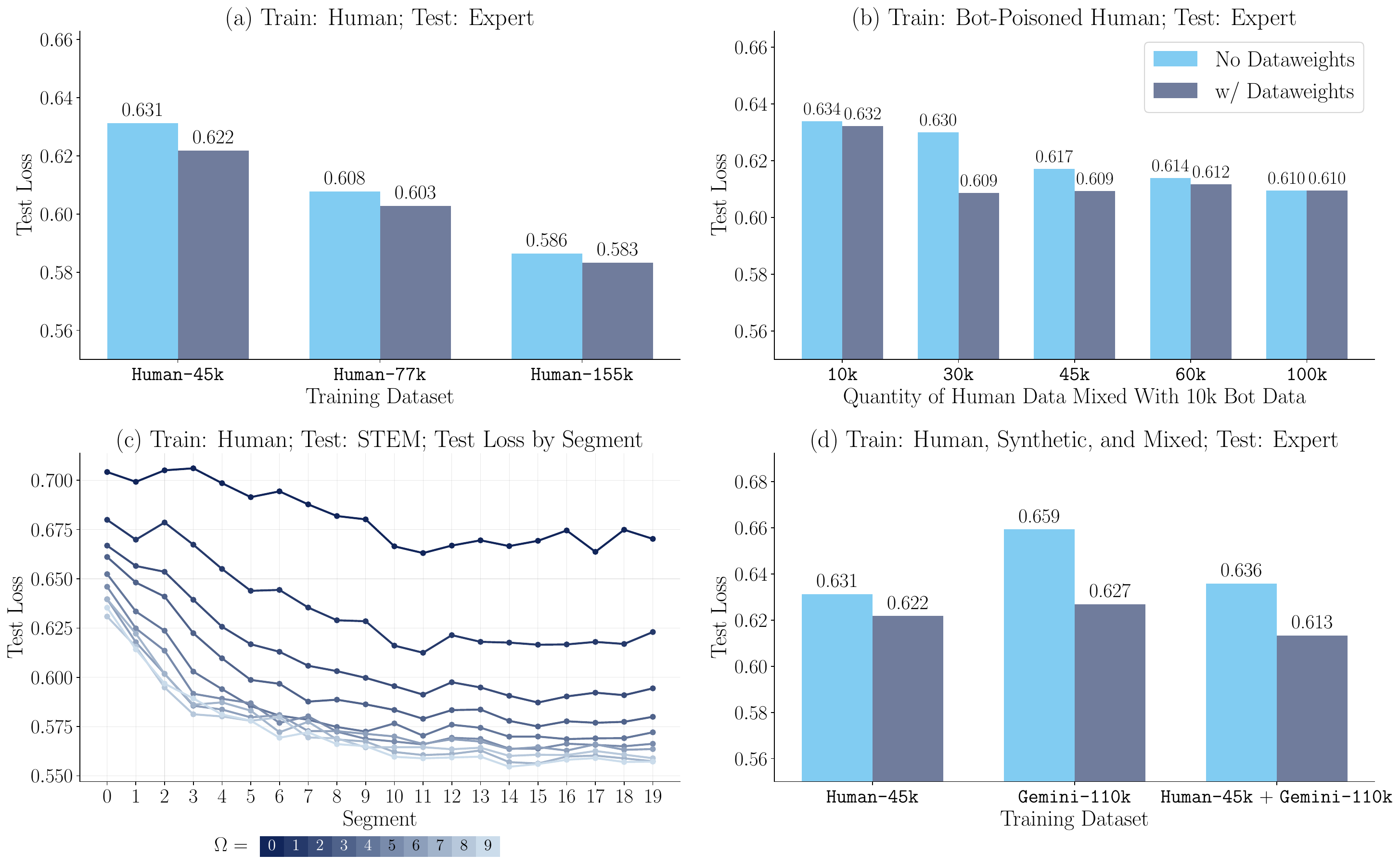}
    \caption{{\em Dataweighting consistently improves reward models across a variety of settings.} Each plot depicts the test loss of reward models trained with and without dataweighting across different training data sources and evaluation sets. In (a), we apply dataweights to a heterogeneous corpus of human preference data, targeting the curated \texttt{Expert} target set, and show that dataweighting improves test-set performance. In (b), we apply dataweighting to a training set comprising heterogeneous
    human data that has been ``poisoned'' with data suspected to be from bots;
    dataweighting allows us to recover performance. In (c), we plot the test loss as a function of the \textcolor{black}{training segment} for dataweights that are further and further optimized for
    the \texttt{STEM} target set. Finally, in (d), we plot the test loss
    of models trained with and without dataweighting on real, synthetic, and mixed
    data; dataweighting enables us to leverage human and synthetic sources to outperform
    either source in isolation.}
    \label{fig:dataweighting}
\end{figure*}

\paragraph{Challenge 1: Leveraging heterogeneous data.} 
Our first main point is that
dataweights can be used to improve in-distribution performance. To demonstrate
this, we train a reward model on varying-sized samples of \texttt{Human-155k}
and show the test loss with and without dataweighting; see
Figure~\ref{fig:dataweighting}(a).
Across all dataset sizes, dataweighting noticeably decreases the test loss on
\texttt{Expert} preferences. 
We provide test and target loss curves in Appendix~\ref{app:heterogeneous-data}.

\paragraph{Challenge 2: Automatically downweighting noisy or adversarial data.} 
Next, we mix human preference data
with noisy and adversarial bot data, varying the human-to-bot data ratio. We
again target the clean \texttt{Expert} preference set.
Figure~\ref{fig:dataweighting}(b) demonstrates that naive training with such
low-quality data can harm performance, especially when the human-to-bot ratio is
low. At the same time, in almost all cases, 
dataweighting restores performance. Plots of test and target training
curves are found in Appendix~\ref{app:noisy-data}, in addition to
dataweight distributions that show how bot
dataweights shift downward during optimization.

\paragraph{Challenge 3: Customizing to specific domains.} 
Next, we demonstrate how dataweighting enables customization
of reward models to specific domains. We target the \texttt{STEM} preference
set, which contains math, science, and coding preference pairs. The annotation of this dataset follows a substantially different
process from the crowdsourced annotation process on Arena.

Without dataweighting, training on \texttt{Human-155k} yields a test loss of
$0.6631$ on the \texttt{STEM} preference set. This makes sense: the target
domain is significantly out-of-distribution with respect to the training data.
However, once we apply dataweighting, our learned dataweights bring the test
loss down to $0.5572$. We provide the test loss ``meta-training'' curves in
Figure~\ref{fig:dataweighting}(c). Each line denotes a new meta-epoch $\Omega$
(i.e., a full run of the algorithmm $\mathcal{A}$). With each meta-epoch, 
we decrease the test loss nearly monotonically. Note that the curve
for $\Omega=0$ corresponds to a normal training procedure where all data points
are weighted equally. The target loss curves, which follow a similar trend to
the test loss curves, can be found in Appendix~\ref{app:customization}.

\paragraph{Challenge 4: Safely incorporating synthetic data.}
We study whether dataweighting allows us to use synthetic LLM-as-a-judge
labels in place of human labels while preserving performance.
Specifically, we train three reward models on \texttt{Human-45k},
\texttt{Gemini-\allowbreak110k}, and \texttt{Human-45k} +
\texttt{Gemini-\allowbreak110k}, respectively. We do this with and without
dataweights. We use unequal dataset sizes to reflect real-world constraints:
synthetic data is cheap to generate at scale, while human-curated data remains
expensive and limited. 

In Figure \ref{fig:dataweighting}(d), we observe that under a standard,
non-dataweighted procedure, purely training on \texttt{Human-45k} gives the best
test loss. Training on only synthetic data, \texttt{Gemini-110k}, yields poor
test loss. Lastly, although mixing synthetic and real data (\texttt{Human-45k} +
\texttt{Gemini-110k}) improves performance over synthetic data alone, it does
not exceed training on only the real human subset. Under the dataweighting
procedure, training on purely synthetic data results in much closer performance
to real human data. Furthermore, with dataweighting, combining both human and
synthetic data together reaches the best performance, better than using the real
human preference set alone. Full test and
target loss curves for \texttt{Human-45k} and \texttt{Gemini-110k} are in
Appendices~\ref{app:heterogeneous-data} and~\ref{app:synthetic-data};
Appendix~\ref{app:synthetic-data}
compares learning curves on synthetic and human data.

\subsection{Parametric dataweights and larger-scale transfer}
\label{sec:parametric-experiments}
So far our dataweights have used the identity parameterization described in Section \ref{sec:dataweighting}. 
The number of learned weights therefore grows with the dataset size $n$, which
might be prohibitively expensive for large training sets; moreover, this rules out direct
transfer of the dataweights to new training datasets.
We now study the parametric variant~\eqref{eq:parametric-w}: a small head
$h_\psi$ maps fixed example features $\phi(z_i)$ to nonnegative weights
$\smash{w_i^\psi=h_\psi(\phi(z_i))}$. Specifically, we define:
\[
    h_\psi(\phi(z))=\mathrm{softplus}(\phi(z)^\top\mathbf{a}+b),
    \qquad \psi=(\mathbf{a},b).
\]
Here, $\phi(z)\in\mathbb{R}^{4608}$ concatenates frozen OpenAI
\texttt{text-embedding-3-small} embeddings of the prompt and the two responses,
each with $1{,}536$ dimensions. The learnable projection $\mathbf{a}$ and scalar
bias $b$ give $p=4{,}609$ parameters, independent of the training dataset size $n$. 

\paragraph{Same-scale comparison.} We extend the customization setup of
Section~\ref{sec:motivating_challenges} by varying how many \texttt{STEM} preference
examples are mixed into a 45k-row human-preference base, with sizes $\{0, 500,
1000, 2000\}$, and optimize the parameterizations across
three random seeds. Figure~\ref{fig:transfer-bars}(a)
shows comparable losses for the per-sample dataweights and the parametric dataweights across different data sizes, despite the linear parameterization having a roughly
tenfold smaller parameter count. 

\begin{figure*}[!ht]
    \centering
    \begin{subcaptiongroup}
        \phantomcaption\label{fig:r50k-size-sweep-seeds}
        \phantomcaption\label{fig:r50k-transfer-bars}
        \phantomcaption\label{fig:h155k-transfer-bars}
    \end{subcaptiongroup}
    \includegraphics[width=.96\linewidth]{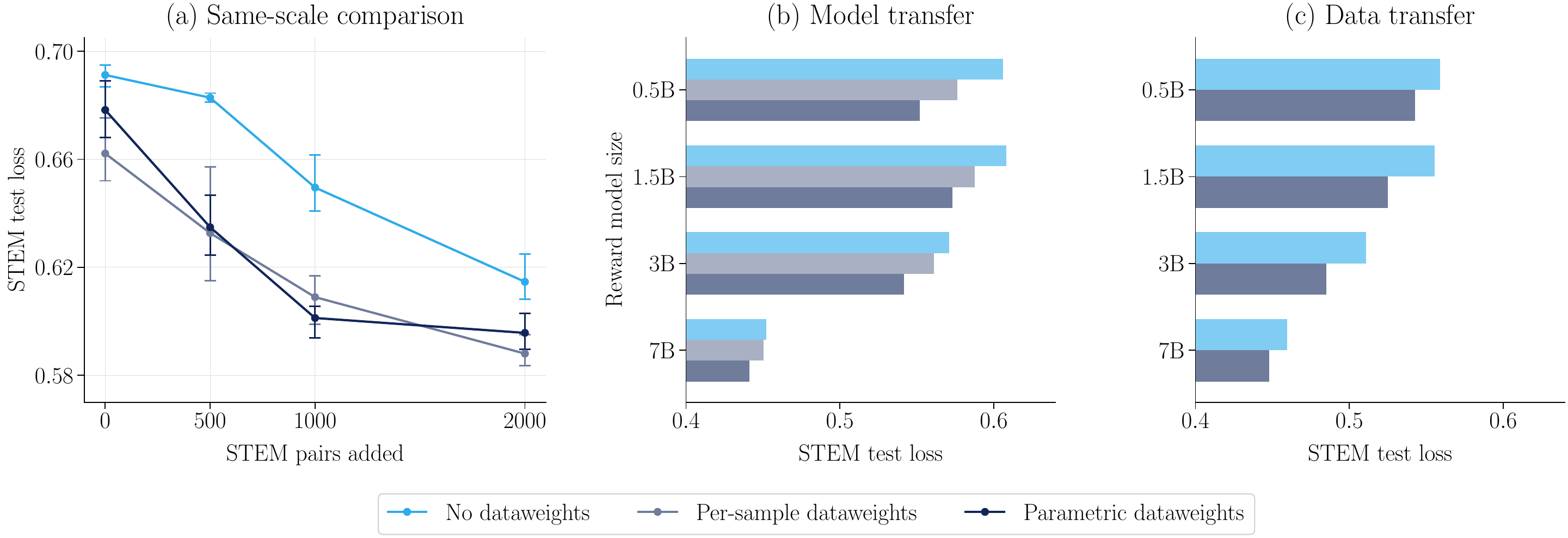}
    \caption{
    {\em Parametric dataweights transfer to larger models and new training data.}
    \emph{(a)} \texttt{STEM} test loss when training on
    \texttt{Human-45k} plus $\{0,500,1000,2000\}$ \texttt{STEM} examples.
    Curves show means over three seeds; error bars show their minimum and maximum.
    \emph{(b)} Transferring dataweights learned at model size 0.5B to 1.5B, 3B, and 7B reward models
    on \texttt{Human-45k} plus $5{,}000$ \texttt{STEM} examples ($50$k total).
    \emph{(c)} Same setup as (b), transferring the linear dataweighting map to the full
    \texttt{Human-155k} training base, with approximately $3.5$ times as many human examples;
    the new examples were not seen while learning the dataweighting~map.}
    \label{fig:transfer-bars}
    \vspace{-3pt}
\end{figure*}

\vspace{4pt}
\paragraph{Cross-scale transfer.}
  A key practical question is whether dataweights learned with a small reward model can be reused when training larger models. We therefore learn dataweights at the 0.5B scale, freeze them, and apply them unchanged when training 1.5B,
  3B, and 7B reward models for a single weighted epoch. This \emph{forward-only} transfer requires no additional metagradient computation at the larger scales. Figure~\ref{fig:transfer-bars}(b) shows that both identity and linear dataweights
  improve over the unweighted baseline across all scales. These results suggest that dataweighting can be amortized: the
  expensive meta-optimization can be performed at a smaller scale, and the resulting weighting rule can be deployed directly in training runs of a larger model.

\paragraph{Generalizing to unseen training examples.}
  A limitation of per-sample dataweights is that they only assign weights to examples
  seen during meta-optimization. The linear parameterization removes this restriction:
  because it maps an example embedding to a weight, it can be evaluated on new training
  examples without relearning a separate scalar for each one. We test this by applying
  a linear head learned on the smaller 45k-row base to the full \texttt{Human-155k}
  training set. This transfer requires only computing embeddings and a forward pass
  through the learned weighting rule, rather than running metagradient optimization on
  the larger dataset. Figure~\ref{fig:transfer-bars}(c) shows that the resulting
  weights continue to improve over the unweighted baseline across model scales. Thus,
  parametric dataweights provide a way to amortize data curation across datasets: a
  weighting rule learned on a manageable subset can be deployed directly on larger pools containing previously unseen training examples.

\subsection{Optimizing a complex target objective}
\label{sec:nontrain-target}

\begin{table}[!ht]
  \centering
  \color{black}
  \captionsetup{
    labelfont={color=black},
    textfont={color=black},
    justification=raggedright,
    singlelinecheck=false
  }
  \small
  \setlength{\tabcolsep}{3pt}

  \begin{minipage}[t]{0.58\linewidth}
    \centering
    \begin{tabular}{@{}llcc@{}}
      \toprule
      Training loss ($\ell$) & Target loss ($L'$)             & Spearman $\rho$ $\uparrow$ & Rank MAE $\downarrow$ \\
      \midrule
      BT & {\color{gray} N/A}      & 0.541          & 7.91          \\
      BT & BT       & 0.565          & 7.67          \\
      Spearman & {\color{gray} N/A} & 0.325          & 9.05          \\
      BT & Spearman    & \textbf{0.625} & \textbf{7.31} \\
      \bottomrule
    \end{tabular}
        \caption{Dataweighting allows for complex target 
        objectives that are difficult to train for directly. Here, 
        standard Bradley-Terry (BT) training with dataweights 
        optimized for downstream ranking correlation outperforms (a)
        unweighted BT training (b) training with dataweights optimized 
        for BT loss on the target dataset; and (c) attempting 
        to train directly for correlation on the training set.
      Depicted graphically in Appendix~\ref{app:ranking-performance}.}
\label{tab:leaderboard-full-test-bars}
  \end{minipage}%
  \hfill
  \begin{minipage}[t]{0.4\linewidth}
    \centering
    \begin{tabular}{@{}lcc@{}}
      \toprule
      Target loss ($L'$)  & Spearman $\rho$ $\uparrow$      & Rank MAE $\downarrow$          \\
      \midrule
      N/A  & 0.83 [0.79, 0.94]               & 9.33 [8.00, 9.88]              \\
      BT      & 0.88 [0.82, 0.94]               & 7.89 [7.33, 8.22]              \\
      Spearman & \textbf{0.95} [0.91, 0.96]      & \textbf{6.17} [4.67, 6.47]     \\
      \bottomrule
    \end{tabular}
    \caption{Re-evaluating the best three methods from Table \ref{tab:leaderboard-full-test-bars} 
    (BT training with no dataweights, dataweights optimized for target BT, and dataweights optimized 
    for target Spearman correlation) with nine model-disjoint holdout splits.
      Medians [25th, 75th percentiles], also depicated graphically in Appendix~\ref{app:ranking-performance}.}
    \label{tab:ranking-distribution-splits}
  \end{minipage}
\end{table}

  So far the target $L'$ has shared the Bradley-Terry form of the training loss.
  We now consider a target objective that is not a training loss: \emph{ranking}
  the response-generating LLMs. In Arena's data, each LLM response is accompanied by the model identity.
  
  For each model $i$, let $s_i(\theta)$ denote the
  average reward assigned by the reward model to responses from model $i$, and let
  $b_i$ denote the ground-truth Bradley-Terry coefficient fit from the original
  pairwise annotations. The evaluation metrics are (a)
  Spearman~$\rho$ between the model-induced ranking $\{s_i(\theta)\}$
  and the ground-truth ranking $\{b_i\}$, as well as (b)
  Rank MAE, the mean absolute error in model rank.
  Because Spearman~$\rho$ is non-differentiable, we use a differentiable
  soft-rank surrogate as $L'$. For a score vector $u$, define
  \begin{equation}
  \label{eq:soft-rank}
      \widetilde{r}_i(u) \;=\; 1 + \sum_{j \neq i}
      \sigma\!\left((u_j - u_i)/\tau\right),
  \end{equation}
  where $\sigma(t)=1/(1+e^{-t})$ is the logistic function and $\tau>0$ is
  the temperature. We take
  $\widetilde{L}'(\theta)$ to be the negative Pearson correlation between
  $\widetilde{r}(s(\theta))$ and $\widetilde{r}(b)$. During meta-optimization,
  this surrogate is used only as the target objective for learning dataweights;
  the reward model itself is still trained with the standard
  (weighted) Bradley-Terry loss.

  \textcolor{black}{Table~\ref{tab:leaderboard-full-test-bars}} compares this rank-targeted
  dataweighting with standard Bradley-Terry training, with and without dataweighting, as well as a baseline that trains the reward model directly using the soft-Spearman
  loss. Direct soft-Spearman training performs worse than Bradley-Terry training,
  suggesting that the rank surrogate is a weak inner-loop learning signal. In
  contrast, using the rank surrogate as a target objective improves both Rank MAE
  and Spearman~$\rho$.
  This suggests a useful separation of roles: Bradley-Terry captures the local signal needed to fit the reward
  model, while the rank surrogate operates one level up, selecting the BT-trained
  trajectory whose induced model ordering best matches the ground-truth downstream
  ranking.

\paragraph{Consistency across held-out splits.}
  To verify that the improvement in \textcolor{black}{Table~\ref{tab:leaderboard-full-test-bars}} is not
  an artifact of a single model split, we repeat the comparison on nine model-disjoint
  holdout splits, each holding out a different set of response-generating models from
  the leaderboard (\textcolor{black}{Table~\ref{tab:ranking-distribution-splits}}). Across splits,
  Spearman-targeted meta-optimization attains a consistently higher held-out
  Spearman~$\rho$ and lower Rank MAE than both unweighted Bradley-Terry training and
  BT-targeted dataweighting, indicating that the rank surrogate selects dataweights
  whose induced ordering generalizes to unseen models.

\paragraph{Transferring parametric dataweights.}
  The per-sample dataweights used above are tied to specific training rows. Following the
  parametric dataweighting of Section~\ref{sec:parametric-experiments}, we instead
  optimize a parametric map from a frozen text embedding of each preference pair to
  its scalar weight, fit on a small subset against the soft-Spearman objective, and then
  apply the learned map to assign weights to all $92$k training pairs.
  We then train a reward model once with those fixed weights, without further
  meta-optimization (Figure~\ref{fig:ranking-transfer-by-size}). A map learned
  on $50$k pairs increases held-out Spearman~$\rho$ from $0.648$ (unweighted) to
  $0.915$, approaching the $0.925$ obtained by non-parametric, per-sample dataweights.
  The gain grows with the size of the subset used to learn the parametric mapping. 
  Because the map depends only on the pair
  embedding, the same weighting can be applied to a different reward model without
  re-running the meta-optimization.

\begin{figure}[!ht]
    \centering
    \begin{minipage}[c]{0.45\linewidth}
        \centering
        \includegraphics[width=\linewidth]{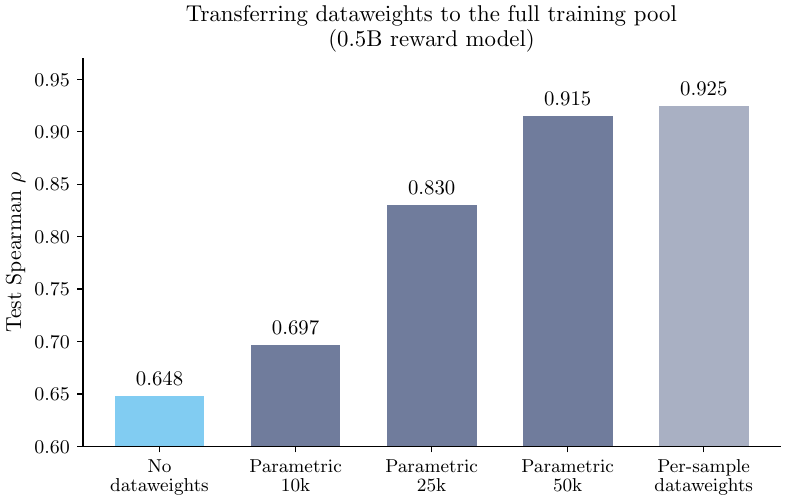}
    \end{minipage}
    \hfill
    \begin{minipage}[c]{0.49\linewidth}
    \caption{{\em Parametric rank-dataweights transfer to unseen training points.} A
    parametric dataweight map is optimized against the
    soft-Spearman objective on a small subset ($10$k/$25$k/$50$k pairs), then used to
    predict weights for all $92$k training pairs; a new reward model is
    then trained once under those weights. Held-out leaderboard Spearman~$\rho$ improves
    with the meta-training size and approaches non-parametric, per-sample dataweights, while
    unweighted training lags well behind.}
    \label{fig:ranking-transfer-by-size}
    \end{minipage}
\end{figure}

\section{Theoretical intuition for dataweighting}
\label{sec:theory}

We provide a few stylized theoretical results illustrating how a small target set
can guide dataweighting. We study linear regression with two signal
features and many weak nuisance features.
For simplicity, we consider a simplified
version of dataweighting that selects the
fixed-size source subset minimizing
the empirical target loss,
sidestepping any potential optimization issues. That is, we assume 
access to an oracle solving:
\begin{equation}
    \label{eq:mgd-oracle}
    \arg\min_{\substack{w \in \{0, 1\}^n \\ \|w\|_0 = k}} \sum_{j=1}^{n'} \ell(z'_j; \widehat{\theta}(w)),
    \qquad \text{ where }
    \qquad \widehat{\theta}(w) \in \arg\min_{\theta} \sum_{i=1}^n w_i \cdot \ell(z_i; \theta).
\end{equation}
Here, $\widehat{\theta}(w)$ represents minimum-norm weighted least-squares regression on the training set 
(i.e., $\mathcal{A}(w, \mathbf{0}_d, \infty)$ in our prior notation).
Note that our empirical results do not rely on such an oracle and instead 
apply (meta-)gradient descent to solve a problem similar to 
\eqref{eq:mgd-oracle} (albeit with continuous, unrestricted weights
and for a more complex learning algorithm $\mathcal{A}$).

We compare this algorithm to a baseline of pretraining on the source data
followed by fine-tuning on the same target set.
To ensure our proofs do not rely on pathologies of overparameterized regression, 
we compare to both pretraining on the entire source dataset ($w = \mathbf{1}_n$)
as well as pretraining on a random $k$-subset of the source data.
In this section, we present the high-level results and corresponding 
intuition, while Appendix~\ref{app:theory} provides the precise assumptions,
statements, and proofs.

\paragraph{Setting and intuition.} We study minimum-norm $d$-dimensional 
linear regression with 
$n$ source data points and $n' < n$ target points, with $n' < d$. 
Both source and target 
points have Gaussian covariates with mean zero and covariance
matrices $\Sigma_S$ and $\Sigma_T$ respectively. 
Target labels are given by 
$$f^*(x) := x_1,$$ 
and so the optimal target parameter is the first canonical vector, $\theta^* = e_1$.
For the first two results below, 
source labels are given by $\hat{f} \neq f^*$, where the two labeling functions 
agree on the halfspace $x_3 \leq 0$ and otherwise disagree by $\delta > 0$. 

We let $\mathcal{R}_{DW}(n, n', k, d)$ be the expected population risk of finding and using the
optimal $k$-subset of the training data using the oracle \eqref{eq:mgd-oracle}.
We then let $\mathcal{R}_{FT}(n, n', k, d)$ denote the minimum expected population risk between 
(a) first fitting on all the data then fine-tuning on the target data;
(b) fitting on $k$ random data points then fine-tuning on the target data.
Since  $n' < d$, the fine-tuning step corresponds to finding the minimum-norm
perturbation to the estimator from pre-training that interpolates the 
target data.

The main idea behind the proofs is 
that we have insufficient target data to learn $f^*(x)$ directly,
and so fine-tuning heavily overfits to the $n'$ target samples.
Using the unweighted source data leads to reliance on spurious features that the
target data is insufficient to unlearn during fine-tuning.
It turns out, however, that the target data suffices to guide least-squares
regression on the {\em source} data towards a good solution.

\paragraph{Correcting biased labels.}
Even arbitrarily small label errors (e.g., from a synthetic labeler) 
can induce a spurious dependence that fine-tuning cannot fully remove. 
In Theorem \ref{thm:theory-bias}, we show that as we take the number of 
source data points $n \to \infty$ the risk of dataweighting converges to 
zero, while the risk of pretraining and fine-tuning stays bounded away 
from zero, formalizing the following informal result:
\begin{informaltheorem}
    \label{infthm:synth}
    Fix any $n' \geq 4$, $\alpha \in (0, \frac{1}{2})$, and $\delta > 0$.
    There exists a family of covariance matrices $(\Sigma_S(d), \Sigma_T(d))$
    and a source label function $\hat{f}(x)$
    and dimension function $d(n) > n + n'$ such that
    \begin{align*}
        \|f^* - \hat{f}\|_\infty &\leq \delta \\
        \lim_{n \to \infty} \mathcal{R}_{FT}(n, n', \lfloor \alpha n \rfloor, d(n)) &\geq c_{n'} \cdot \delta^2 \\
        \lim_{n \to \infty }\mathcal{R}_{DW}(n, n', \lfloor \alpha n \rfloor, d(n)) &= 0,
    \end{align*}
    where $c_{n'} > 0$ is a positive constant depending only on $n'$.
\end{informaltheorem}

The key idea is that upweighting only the examples within the halfspace 
$x_3 \leq 0$ (i.e., the examples on which the labeling functions agree)
yields an estimator with zero asymptotic risk.
Even a finite-size validation set suffices to select such examples 
(or examples whose population risk is comparable), 
and so dataweighting attains zero asymptotic risk. 
Meanwhile, unweighted training on the source data (whether $n$ samples 
or $k = \alpha n$ samples) results in an estimator 
that places non-trivial weight on components other than $e_1$, and 
there is not enough target fine-tuning data to unlearn these spurious 
features.

\paragraph{Filtering severe corruption.}
Extending the previous result, we can show that even at fixed sample size $n$ 
and dimension $d$, and as the label corruption size 
$\smash{\delta := \|f - \hat{f}\|}$ grows, 
dataweighting identifies an entirely clean subset (i.e., a subset on which 
$\hat{f}(x) = f^*(x)$) with probability tending to one,
conditional on enough clean examples existing. 
The probability of error is
exponentially small in the source pool size.
In particular, Theorem~\ref{thm:theory-filtering} formalizes the following 
informal result:
\begin{informaltheorem}
    Fix $n, n', d$ with $n' > 1$ and $d > n + n'$.
    Let $\hat{w}(n, n', k, d)$ be the output of the dataweighting oracle
    \eqref{eq:mgd-oracle}, and take the source and target labels $\hat{f}$
    and $f^*$ from Informal Theorem \ref{infthm:synth}, parameterized by 
    their discrepancy $\delta$. Then, for $\alpha \in (0, \nicefrac{1}{2})$, 
    the probability that dataweighting 
    assigns positive weight to a training data point on which $f^*$ and 
    $\hat{f}$ disagree decays as:
    \[
        \lim_{\delta \to \infty} \mathbb{P}\left(
            \exists\ i: f^*(x^{(i)}) \neq \hat{f}(x^{(i)}) \text{ and }
            \hat{w}(n, n', \lfloor\alpha n\rfloor, d)_i > 0
        \right) \leq \exp\{-n(\nicefrac{1}{2}-\alpha)^2\}.
    \]
\end{informaltheorem}
The main idea behind the proof is that as $\delta \to \infty$, the target 
loss of a sample containing even a single point where  
$f^*$ and $\hat{f}$ disagree diverges with high probability.

\paragraph{Adapting to covariate shift.}
Finally, we show that even when the source data is labeled correctly, 
selecting examples can change the learner's inductive
bias to favor features important in the target domain. 
To show this, we study the case where the source data is labeled with 
the true labeling function $f^*$, but the source and target covariances 
vary systematically. In particular, Theorem~\ref{thm:theory-covariate} 
formalizes the following informal theorem:
\begin{informaltheorem}
    Fix $\alpha \in (0, \nicefrac{1}{2})$ and $n' \geq 4$. 
    There exist source and target covariances $(\Sigma_S(n, d),
    \Sigma_T(n, d))$ and a dimension function $d(n)$ such that,
    even with $\hat{f} = f^* = x_1$,
    \begin{align*}
        \lim_{n \to \infty} \mathcal{R}_{FT}(n, n', \lfloor \alpha n \rfloor, d(n))
        &\geq c_\alpha > 0, \\
        \lim_{n \to \infty} \mathcal{R}_{DW}(n, n', \lfloor \alpha n \rfloor, d(n))
        &= 0,
    \end{align*}
    where $c_\alpha > 0$ is a positive constant depending only on $\alpha$.
\end{informaltheorem}
\noindent Note that the construction in Appendix \ref{app:theory} yields this statement after an 
orthogonal change of basis.

\section{Related work}
\label{sec:related}
We discuss prior approaches to leveraging heterogeneous and noisy data sources, particularly focusing on reweighting methods and aggregating heterogeneous preferences.
Closest to our work are papers on automatic data selection
using data attribution methods \citep{ghorbani2019data,xia2024less,xie2023doremi}
and metagradient descent \citep{engstrom2025optimizing,calian2025datarater,dass2025datamil},
as well as prior work on learning to reweight data \citep{ren2018learning}.

\paragraph{\textcolor{black}{Reweighting methods.}} Our dataweights approach
relates to a number of existing studies proposing reweighting strategies for
deep learning~\citep{huang2006correcting, ren2018learning, byrd2019effect,
lipton2018detecting, wang2021importance}, often motivated by particular models
of distribution shift such as covariate shift and label shift. Most related is
the work of \citet{ren2018learning}, who similarly assume a clean target set and
aim to optimize training data weights, although
their update differentiates through a single SGD step.
Also relatedly, Meta-Weight-Net~\citep{shu2019meta}
and DataRater~\citep{calian2025datarater} learn parametric weighting functions
using meta-objectives. Our implementation computes exact metagradients within
long training segments (typically 5--10\% of total training), using scalable
attribution methods that support differentiation through general optimizers.
As described in Section~\ref{sec:dataweighting}, these segment-level updates
approximate differentiation through the full training run.

Despite mixed evidence on the effectiveness of reweighting~\citep{wang2021importance,sagawa2019dro},
we find it effective for model customization and for combining human and
synthetic preferences. We also study how a parametric weighting function transfers
to unseen examples and larger reward models, and how dataweights can be optimized
for downstream ranking objectives rather than the standard Bradley-Terry loss.

\paragraph{\textcolor{black}{Data attribution.}} Data attribution methods \citep{koh2017understanding, ghorbani2019data, ilyas2022datamodels, grosse2023studying, park2023trak} seek to quantify the influence of individual training data points on the model's downstream behavior.
Compared with methods that use attribution scores or metagradients to select
training examples~\citep{engstrom2025optimizing,dass2025datamil,xia2024less},
we optimize continuous example weights directly in the training loss.
We build on the \textsc{Replay} algorithm \citep{engstrom2025optimizing,ilyas2025magic} to compute exact ``metagradients'' of the model's performance on the clean target data with respect to dataweights, allowing us to automatically discover---and thus downweight---noisy and irrelevant parts of the training data.

\paragraph{\textcolor{black}{Distributional robustness and domain adaptation.}}
A central issue with heterogeneous data is that training examples may come from a mixture of domains that differ from the test distribution. A large body of work studies how to learn predictors that robustly generalize under such shifts via domain adaptation \citep{ben2006analysis, ben2010theory, zhao2018adversarial}, distributionally robust optimization (DRO) that optimizes worst-case performance over groups \citep{sagawa2019dro}, and invariance-based objectives such as invariant risk minimization \citep{arjovsky2019invariant}.
Rather than seeking a single representation that is invariant across environments, we directly use a small, high-quality target set to reweight heterogeneous training data toward the target distribution. Recent work in language model pretraining similarly uses robust or minimax objectives to tune mixture weights over data domains (as opposed to per-example weights) to improve downstream performance \citep{xie2023doremi}.

\paragraph{\textcolor{black}{Robust statistics.}}
Heterogeneity is often caused by corrupted observations---e.g., mislabeled examples or adversarial outliers. Robust statistics formalizes this through contamination models and offers robust estimators that take the contamination into account \citep{huber1964robust, diakonikolas2019robust}. A parallel literature in machine learning studies learning with label noise \citep{natarajan2013learning, lidividemix}. All these works aim to detect corruption from internal inconsistency; our work is different in that we exploit an additional, high-quality target sample that specifies the desired test distribution, and we use it to determine how much each training example should influence the learned model.

\paragraph{\textcolor{black}{Synthetic data.}}
An increasingly common cause of heterogeneity is mixing real and synthetic data. Synthetic data has proven useful for improving power in statistical inference \citep{angelopoulos2023prediction, huang2020catalytic}, training better machine learning models in domains such as robotics \citep{tobin2017domain} and computer vision \citep{zhu2017unpaired}, and improving efficiency in language model evaluations \citep{zheng2023judging}. A key challenge with synthetic data is knowing which examples are useful and which are potentially even harmful. Our method can be seen as learning how to weight each synthetic example based on how it improves performance on the target set.

\paragraph{\textcolor{black}{Reward modeling, RLHF, heterogeneous preferences.}} Reward modeling and reinforcement learning from human feedback (RLHF) operationalize ``alignment'' by learning a reward function from human-labeled comparisons and then optimizing a policy against that reward \citep{christiano2017deep, ouyang2022training}. Humans may provide noisy feedback, or even vote adversarially. By optimizing against a clean target dataset, our method downweights such noisy and adversarial votes. At a high level, our work relates to social choice theory \citep{brandt2016handbook, fischbacher2006heterogeneous, ahn2003heterogeneous}, which studies principled ways for aggregating heterogeneous preferences. Some recent works proposed data selection algorithms for reward model training~\citep{feng2025pilaf,shenactive}; however, these approaches primarily focus on reducing data volume.
In the alignment context, recent work formalizes heterogeneous preferences through latent user types and analyzes what it means to align a single policy when reward functions vary across users \citep{shirali2025direct}.

  \begingroup
  \emergencystretch=2em
  \printbibliography
  \endgroup

\clearpage
\endgroup
\endgroup
\setcounter{topnumber}{2}
\setcounter{bottomnumber}{1}
\setcounter{totalnumber}{3}
\appendix
\begingroup
\captionsetup{font=small,skip=4pt}
\setlength{\intextsep}{8pt plus 2pt minus 2pt}

\clearpage
\section{Full algorithm}
We describe our end-to-end algorithm for training 
models from heterogenous data. 
Fix a training algorithm $\mathcal{A}$, 
a training set of $n$ samples and a target set of $n'$ samples.
Putting all the ideas from Section \ref{sec:dataweighting} together, our 
final procedure for learning dataweights is as follows:
\begin{enumerate}
    \item Select a proxy algorithm $\mathcal{A}'$ 
    that runs on a subset $\tilde{n} \leq n$ of the training data; 
    a number of segments $K$;
    a target function $L'$ matching our downstream objective;
    and a dataweight parameterization $h_\psi(\cdot)$
    \item Initialize dataweighting parameters $\psi^{(0)}$ and model parameters $\theta^{(0)}$
    \item Initialize the meta-epoch count $\Omega = 0$
    \item Repeat for $t = 1,\ldots$ until convergence:
    \begin{enumerate}
        \item Compute $w$ from $\psi^{(t-1)}$ using \eqref{eq:parametric-w} or \eqref{eq:np-w}.
        \item Set $\theta^{(t)} \gets \mathcal{A}'(w, \theta^{(t-1)}, \frac{T}{K})$ if $t \mod K \neq 0$, otherwise set $\theta^{(t)} = \theta_0$ and increment $\Omega$
        \item Compute the metagradient $\nabla_w L'(\mathcal{A}'(w, \theta^{(t-1)}, \frac{T}{K}))$
        \item Compute the metagradient with respect to $\psi^{(t-1)}$ via \eqref{eq:parametric-chain-rule} 
        \item Update $\psi^{(t)}$ with a first-order method (e.g., SignSGD)
    \end{enumerate}
    \item (If applicable) Compute dataweights for all $n$ training samples via $w_i := h_\psi(z_i)$ 
    \item (If applicable) Compute and return the final model $\mathcal{A}(w, \theta^{(0)}, T)$
\end{enumerate}

\section{\textcolor{black}{Dataset construction and reward model training}}
\label{app:experimental-setup}

\paragraph{Training data.} We use two public preference datasets, \texttt{arena-\allowbreak human-\allowbreak preference-\allowbreak 100k} and \texttt{arena-\allowbreak human-\allowbreak preference-\allowbreak 140k} \cite{chiang24chatbot}. \textcolor{black}{From these datasets, we filter} for single-turn, non-tie examples, yielding a combined dataset of 155,577 rows of human-labeled pairwise preference data---we hereby refer to this dataset as \texttt{Human-155k}. We also obtained $10,000$ examples of bot and adversarial activity flagged by Arena in their preference data, henceforth referred to as \texttt{Bot-10k}.

Using the prompts and response pairs found in \texttt{Human-155k}, we generate a synthetic preference dataset by replacing the human label with an LLM-as-a-judge label \cite{zheng2023judging}. Specifically, we prompt Gemini 3 Flash to pick the best response given the prompt and response pair.
Measured against the original human labels, the LLM-as-a-judge labels reached an accuracy rate of $0.6812$. We will refer to this synthetic dataset as \texttt{Gemini-155k}.

From \texttt{Human-155k} and \texttt{Gemini-155k} we randomly subsample to derive the dataset pairs (\texttt{Human-45k}, \texttt{Gemini-110k}) and (\texttt{Human-77k}, \texttt{Gemini-77k}). Sampled data rows are disjoint within each pair. The first dataset pair illustrates a setting in which collecting synthetic labels is much cheaper than collecting real human labels, while the second dataset pair keeps the sample size fixed so we isolate the effect of using noisy labels.

\paragraph{Target datasets.} We evaluate dataweighting on two different target datasets: \texttt{Expert} and \texttt{STEM}.
\texttt{Expert} is sourced from Arena's \texttt{arena-expert-5k}: a heavily filtered subset of Arena's preference data aimed at identifying sophisticated, domain-specific prompts. Around $5.5\%$ of Arena's preference data is labeled as ``expert.'' Since \texttt{Expert} is a subset of Arena's general human preference data, we use it to evaluate in-distribution performance gains of dataweighting, albeit a niche part of the distribution.

\texttt{STEM} is sourced from both HelpSteer3 and Tulu2.5 preference datasets \citep{wang2025helpsteer3preferenceopenhumanannotatedpreference, ivison2024unpacking}. From HelpSteer3, we filter for only rows labeled as ``stem'' or ``code'' and for rows where the responses were not tied. These preference rows contain a prompt, two responses, and an overall preference score labeled by ScaleAI annotators noted to have a relevant degree and/or software engineering work experience \citep{wang2025helpsteer3preferenceopenhumanannotatedpreference}. From Tulu2.5, we filter for the \texttt{prm800k\_pairs\_phase2} subset, where chosen and rejected preference pairs have been selected based on mathematical correctness \citep{ivison2024unpacking, lightman2023lets}. Since \texttt{STEM} annotations come from ground truth verifiers and/or hired human rubric labelers with relevant domain expertise, rather than crowdsourced organic usage like in \texttt{Human-155k}, we use this dataset to evaluate out-of-distribution performance of dataweighting.

\textcolor{black}{For each domain, we use a meta-optimization target split and a held-out test split, in addition to the separate checkpoint selection set described below. The target split supplies the objective in Section~\ref{sec:dataweighting}; the test split is used only for evaluation.}
\texttt{Expert} has a target/test split of $3000$/$547$. \texttt{STEM} has a target/test split of $6796$/$6796$.

\paragraph{Reward model training procedure.} We adopt a standard reward model training procedure where we minimize the Bradley-Terry objective, as discussed in Section \textcolor{black}{\ref{sec:results}}. Unless specified otherwise, the base pre-trained model we start from is \texttt{Qwen2.5-0.5B-Instruct}. \textcolor{black}{The 1.5B experiments use \texttt{Qwen2.5-1.5B-Instruct}; the transfer experiments also include 3B and 7B reward models (Section~\ref{sec:parametric-experiments}).} When training a reward model, we remove the \texttt{vocab\_head} and replace it with a linear \texttt{reward\_head} that takes in the hidden dimension output from a \texttt{<|cls|>} token and outputs a scalar reward.

\textcolor{black}{We use the same separate $100$-example checkpoint selection set
across methods for early stopping. This simple protocol lets us evaluate
dataweighting; we leave alternatives such as cross-validation to future work.}

\Needspace{12\baselineskip}
\subsection{LLM-as-a-judge prompt}
\label{appendix:llm-judge-prompt}

We detail the prompt used for synthetic labels below. Synthetic labels were generated by \textcolor{black}{Gemini 3 Flash} with default settings.
\begin{quote}
\ttfamily
\raggedright
You will be shown a user prompt followed by two responses generated by AI assistants.
Your task is to determine which response better satisfies the user's intent.

Evaluate the responses from the *user's perspective*, focusing on the following criteria:

\begin{itemize}
    \item \textit{Helpfulness}: Does the response effectively assist or inform the user?
    \item \textit{Relevance}: Does it directly address the user's question or request?
    \item \textit{Accuracy}: Are the facts and reasoning correct?
    \item \textit{Depth and Detail}: Does it provide sufficient explanation or insight?
    \item \textit{Creativity} (when applicable): Is the response original or engaging?
\end{itemize}

Important guidelines: Your evaluation must be impartial and based solely on quality and alignment with the user's needs.

Return your decision for the winner, either A or B in the tags <winner> ... </winner>.
\end{quote}

\clearpage
\section{\textcolor{black}{Additional ranking experiments}}
\label{app:ranking-followups}

This appendix provides graphical versions of the ranking comparisons in
Section~\ref{sec:nontrain-target} and examines alternative rank surrogates.
The parametric ranking-transfer experiment is reported in
Figure~\ref{fig:ranking-transfer-by-size} in the main text.

\Needspace{0.35\textheight}
\subsection{Ranking performance on held-out models}
\label{app:ranking-performance}

\begin{figure}[H]
    \centering
    \includegraphics[width=0.94\linewidth]{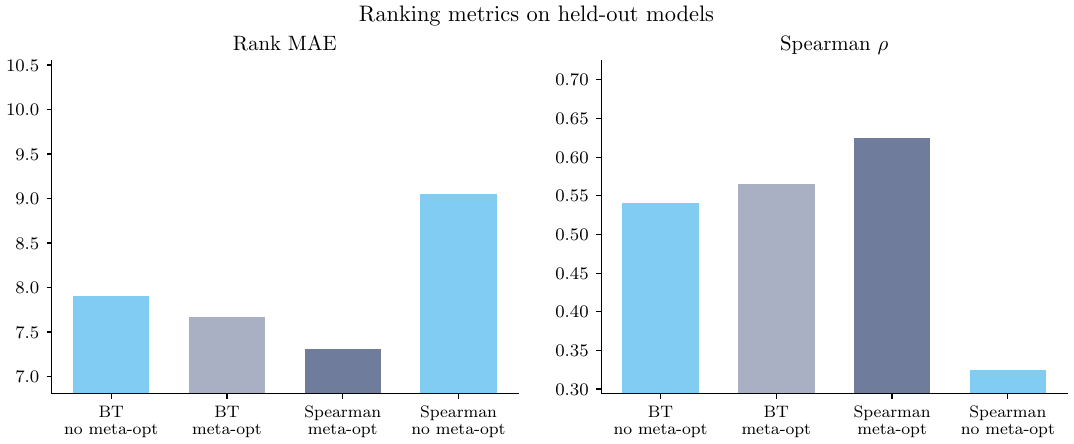}
    \caption{{\em BT and Spearman: standard vs.\ meta-optimized reward models.} \textcolor{black}{Rank MAE (lower is better) and Spearman~$\rho$ (higher is better) on held-out test models.}}
    \label{fig:leaderboard-full-test-bars}
\end{figure}

\begin{figure}[H]
    \centering
    \includegraphics[width=0.76\linewidth]{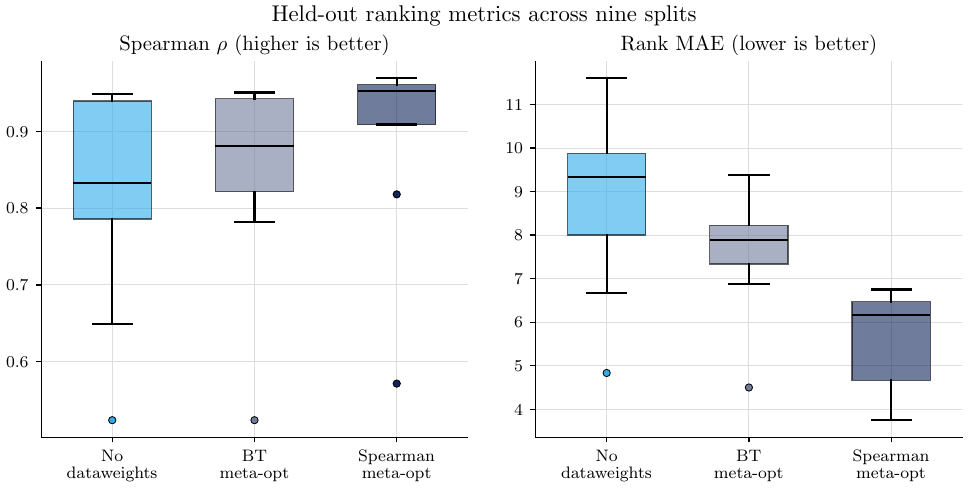}
    \caption{{\em Rank-targeted meta-optimization improves held-out ranking across
    splits.} Held-out Spearman~$\rho$ (left, higher is better) and Rank MAE (right,
    lower is better) across nine model-disjoint holdout splits, comparing unweighted
    Bradley-Terry training, BT-targeted meta-optimization, and Spearman-targeted
    meta-optimization. The Spearman-targeted objective shifts the whole distribution
    toward better rankings (higher~$\rho$, lower MAE) relative to both baselines.}
    \label{fig:ranking-distribution-splits}
\end{figure}

\subsection{Alternative rank surrogates}
\label{app:ranking-surrogates}

  The soft-Spearman surrogate is one of several differentiable choices for the target
  objective. Replacing it with a smooth footrule distance (a smooth-$\ell_1$ penalty on
  the soft ranks, aligned with Rank MAE) performs comparably---matching or slightly
  exceeding soft-Spearman on held-out~$\rho$ (e.g.\ $0.965$ vs.\ $0.957$ on one split),
  consistent with its closer alignment to the MAE metric---whereas a soft-Kendall-$\tau$
  surrogate is consistently weaker. We therefore use soft-Spearman by default.

\clearpage
\section{Additional results}
\label{app:additional-results}

\Needspace{0.52\textheight}
\subsection{Leveraging heterogeneous data}
\label{app:heterogeneous-data}

\begin{figure}[H]
    \centering
    \includegraphics[width=0.54\linewidth]{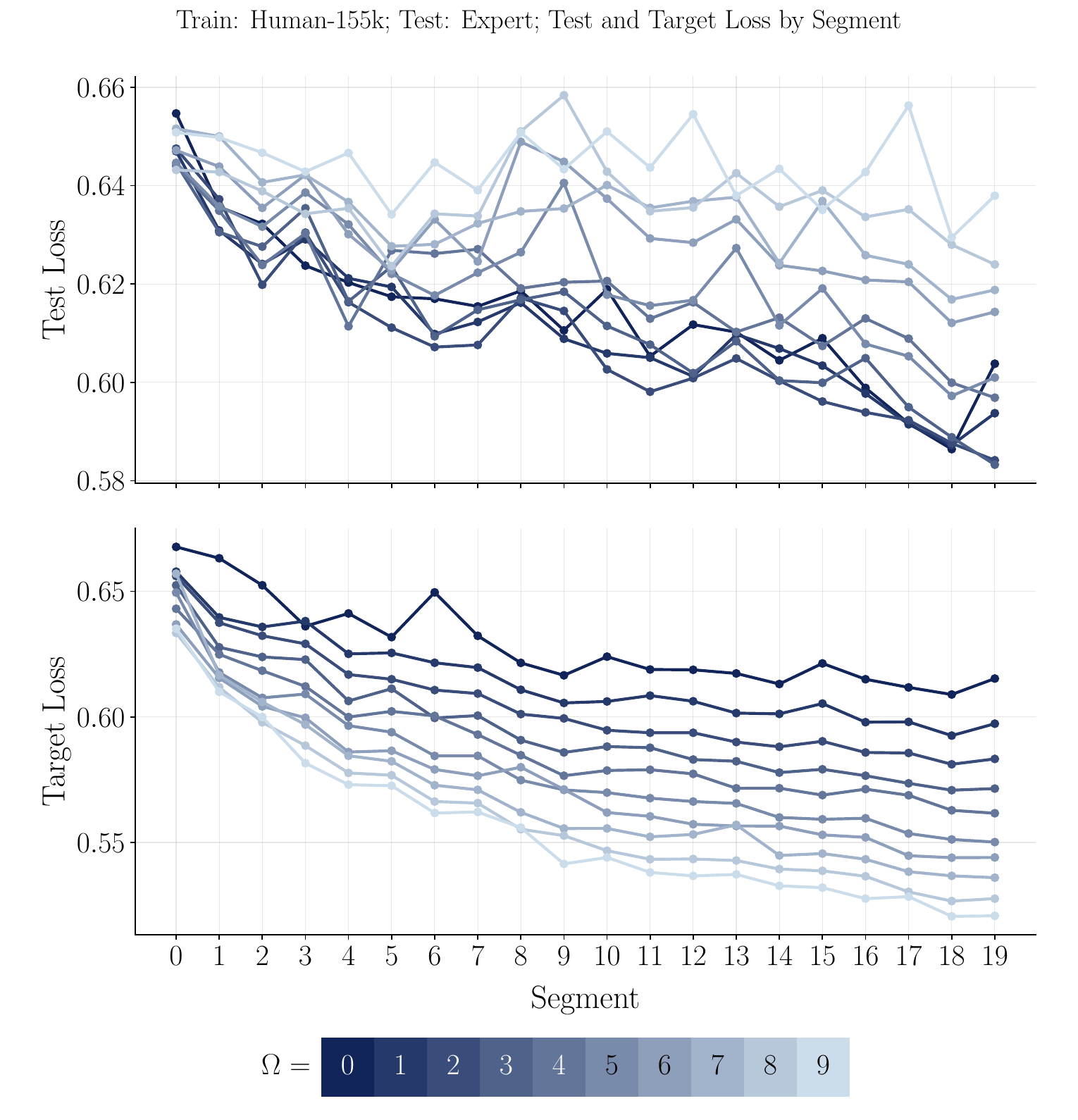}
    \caption{Test and target loss by segment for model trained on \texttt{Human-155k}, targeting \texttt{Expert}. This example shows notable overfitting. \textcolor{black}{Although the first few metagradient steps improve test loss, further steps overfit the target.}}
    \label{fig:human-155k-expert}
\end{figure}

\begin{figure}[H]
    \centering
    \includegraphics[width=0.54\linewidth]{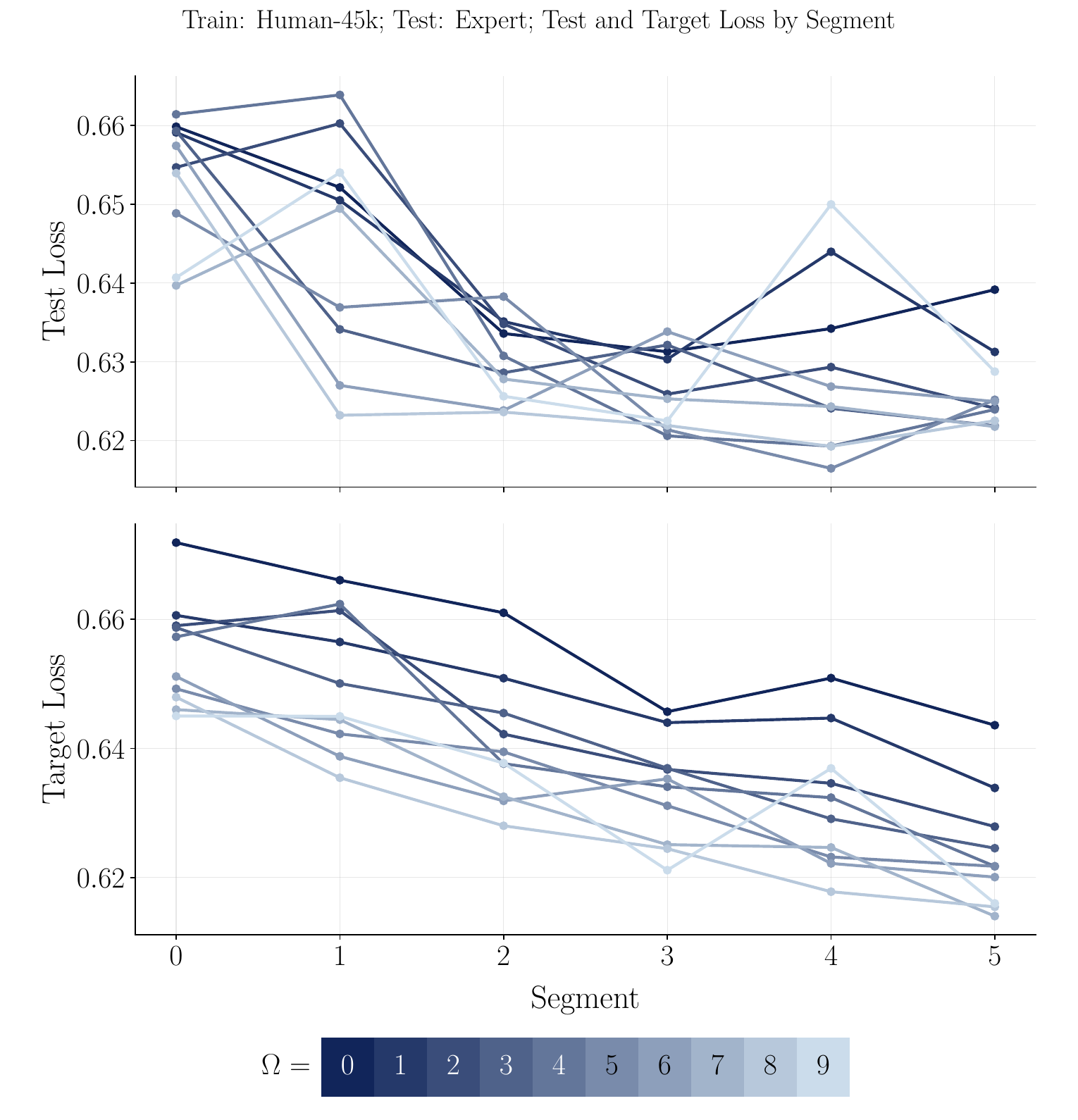}
    \caption{Test and target loss by segment for model trained on \texttt{Human-45k}, targeting \texttt{Expert}.}
    \label{fig:human-45k-expert}
\end{figure}

\begin{figure}[H]
    \centering
    \includegraphics[width=0.54\linewidth]{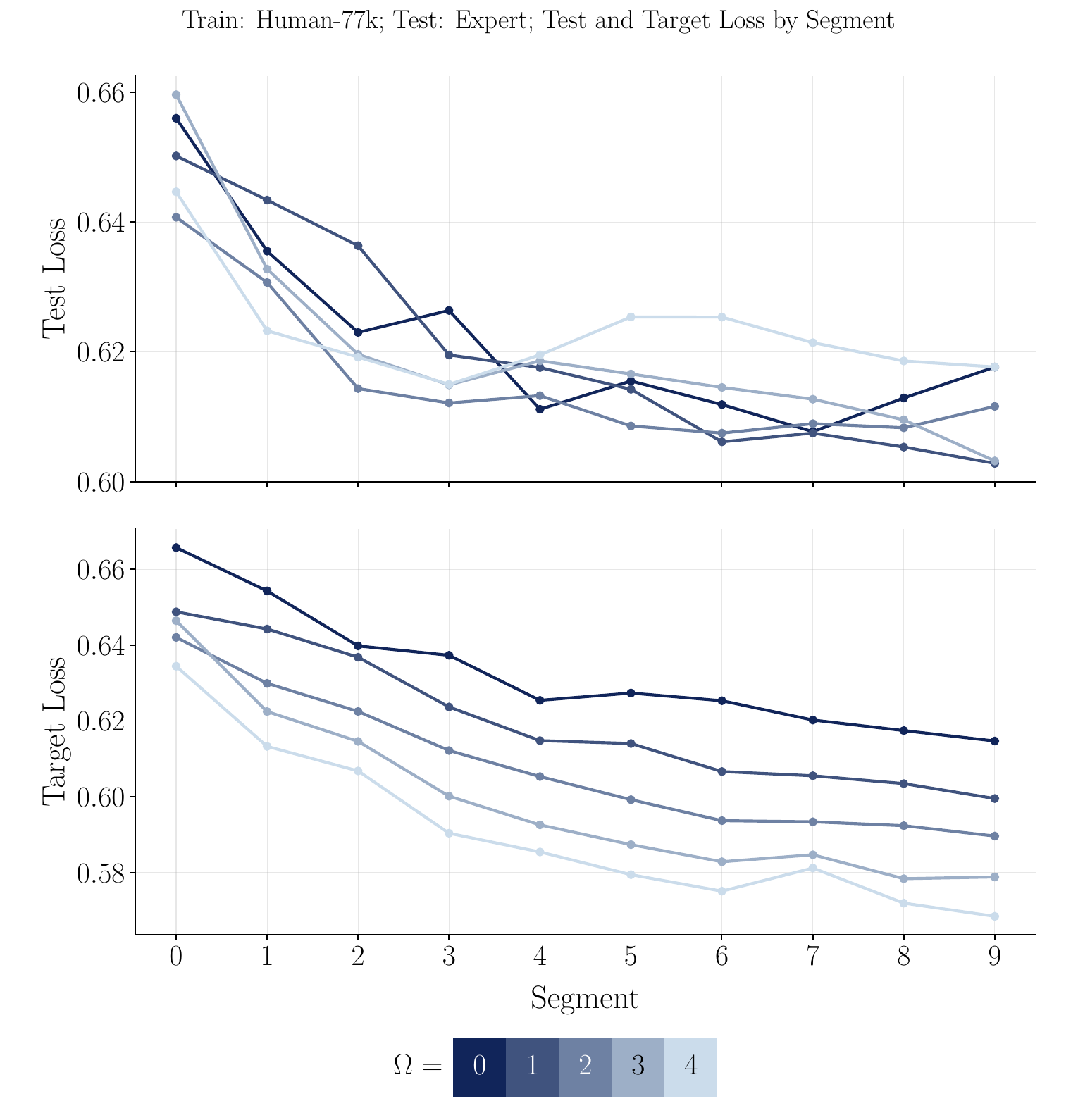}
    \caption{Test and target loss by segment for model trained on \texttt{Human-77k}, targeting \texttt{Expert}.}
    \label{fig:human-77k-expert}
\end{figure}

\Needspace{0.42\textheight}
\subsection{Downweighting noisy and adversarial data}
\label{app:noisy-data}

\paragraph{Evolution of bot and human dataweights.}
In Figure~\ref{fig:dataweight-evolution-30k}, we plot the distribution of
dataweights for bot and human datapoints. We observe the dataweight distribution
for bots shifting down throughout the dataweighting process.

\begin{figure}[H]
    \centering
    \includegraphics[width=0.82\linewidth]{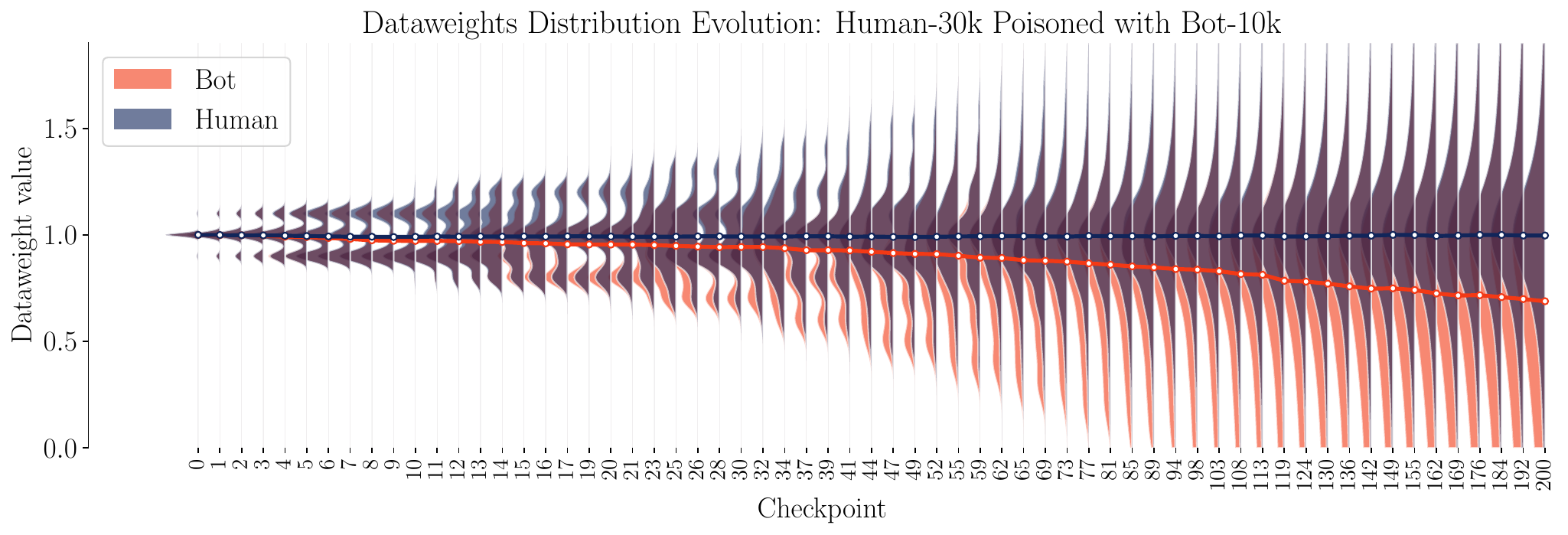}
    \caption{{\em Dataweighting automatically separates bot and human data.} Here, we visualize a KDE ridge plot of the learned dataweight distribution for human and bot data points throughout the dataweight learning procedure. \textcolor{black}{Each horizontal ridge shows a kernel density estimate of the dataweights at one checkpoint.} Bot data points decrease in weight over time. }
    \label{fig:dataweight-evolution-30k}
\end{figure}

\begin{figure}[H]
    \centering
    \includegraphics[width=0.62\linewidth]{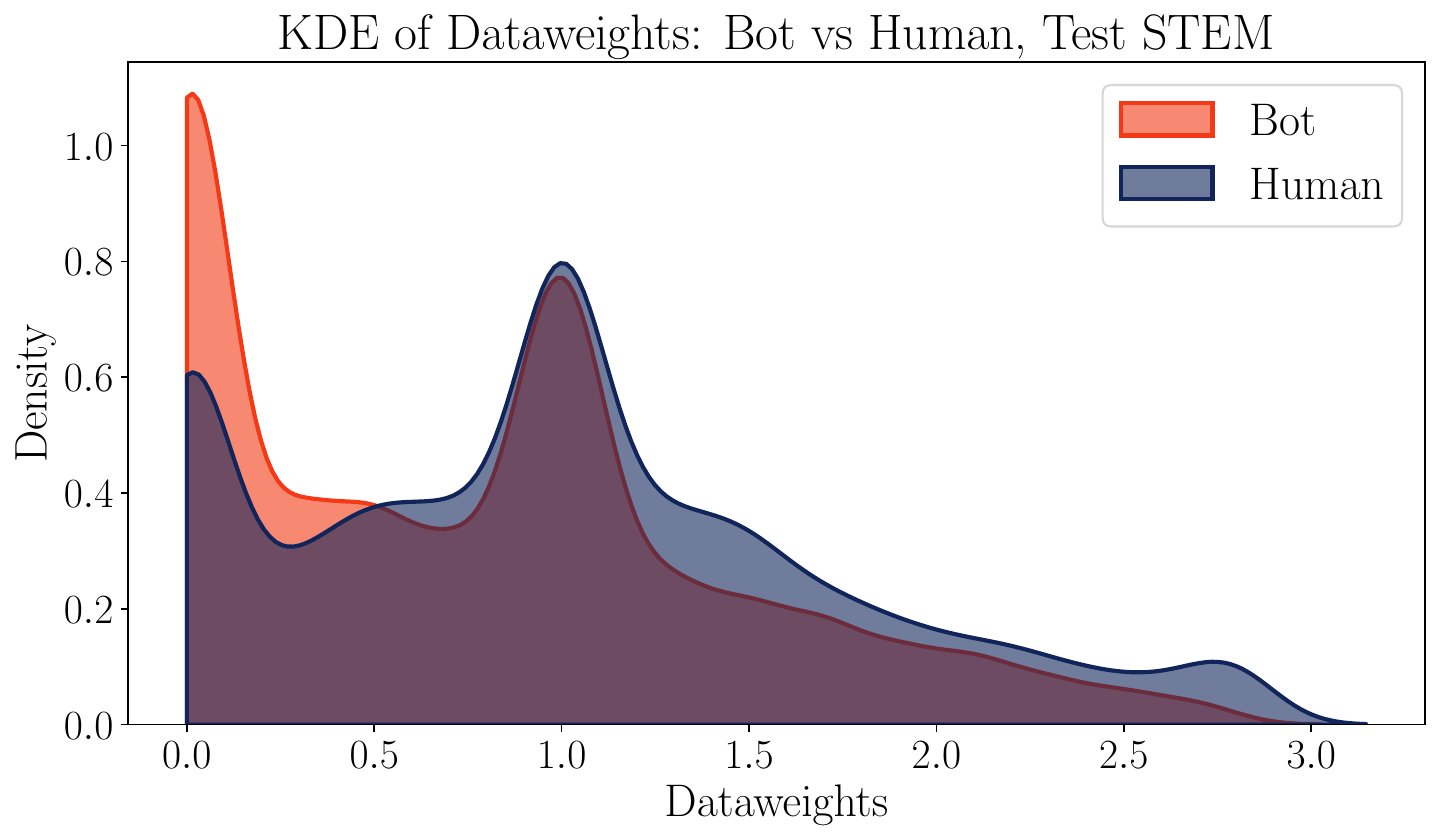}
    \caption{Kernel density estimate of \textcolor{black}{learned dataweights} for \texttt{Bot} versus \texttt{Human} examples, targeting \texttt{STEM}.}
    \label{fig:kde-dataweights-bot-human-stem}
\end{figure}

\begin{figure}[H]
    \centering
    \includegraphics[width=0.54\linewidth]{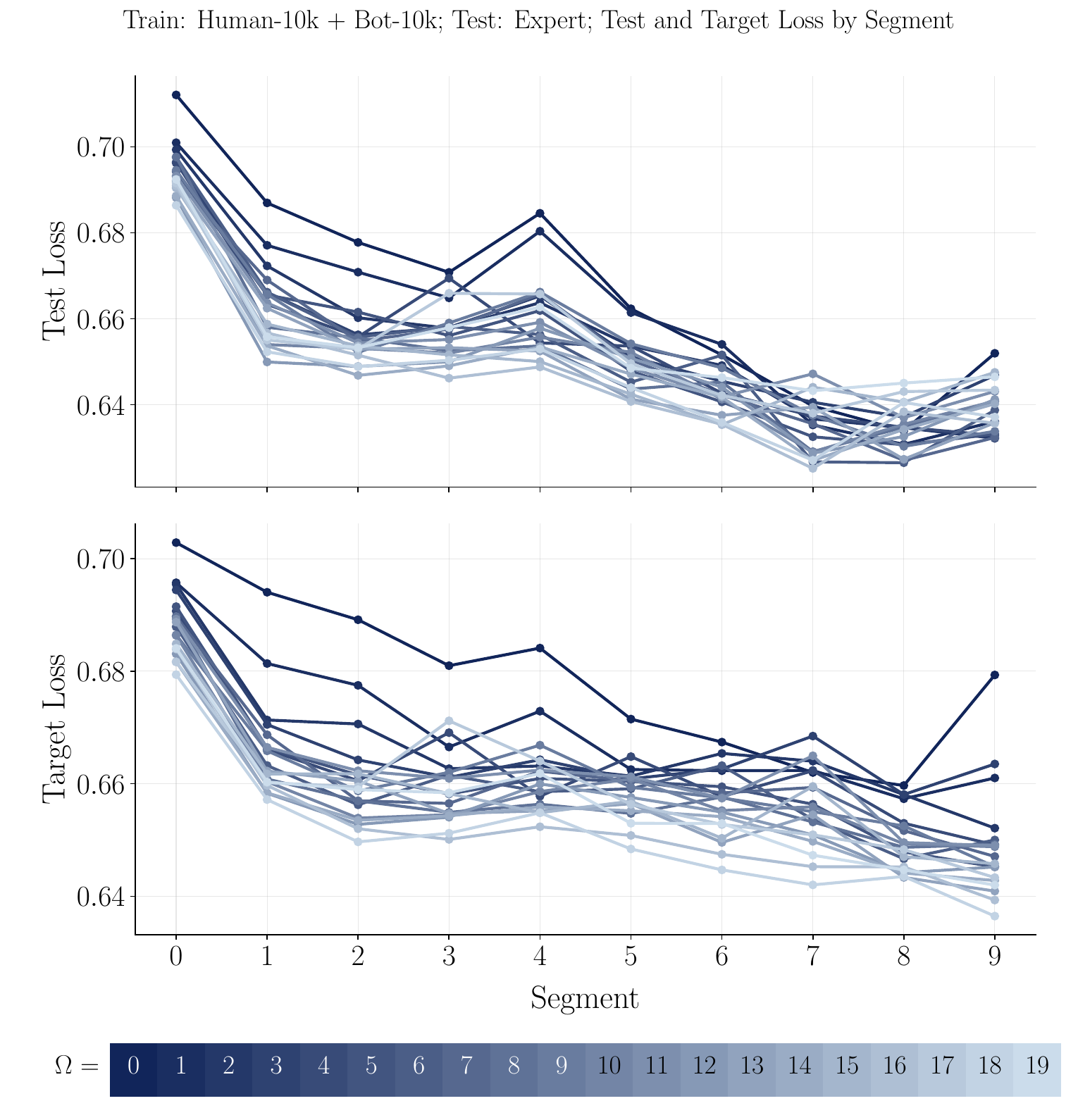}
    \caption{Test and target loss by segment for model trained on \texttt{Human-10k} + \texttt{Bot-10k}, targeting \texttt{Expert}.}
    \label{fig:human-10k-bot-10k-expert}
\end{figure}

\begin{figure}[H]
    \centering
    \includegraphics[width=0.54\linewidth]{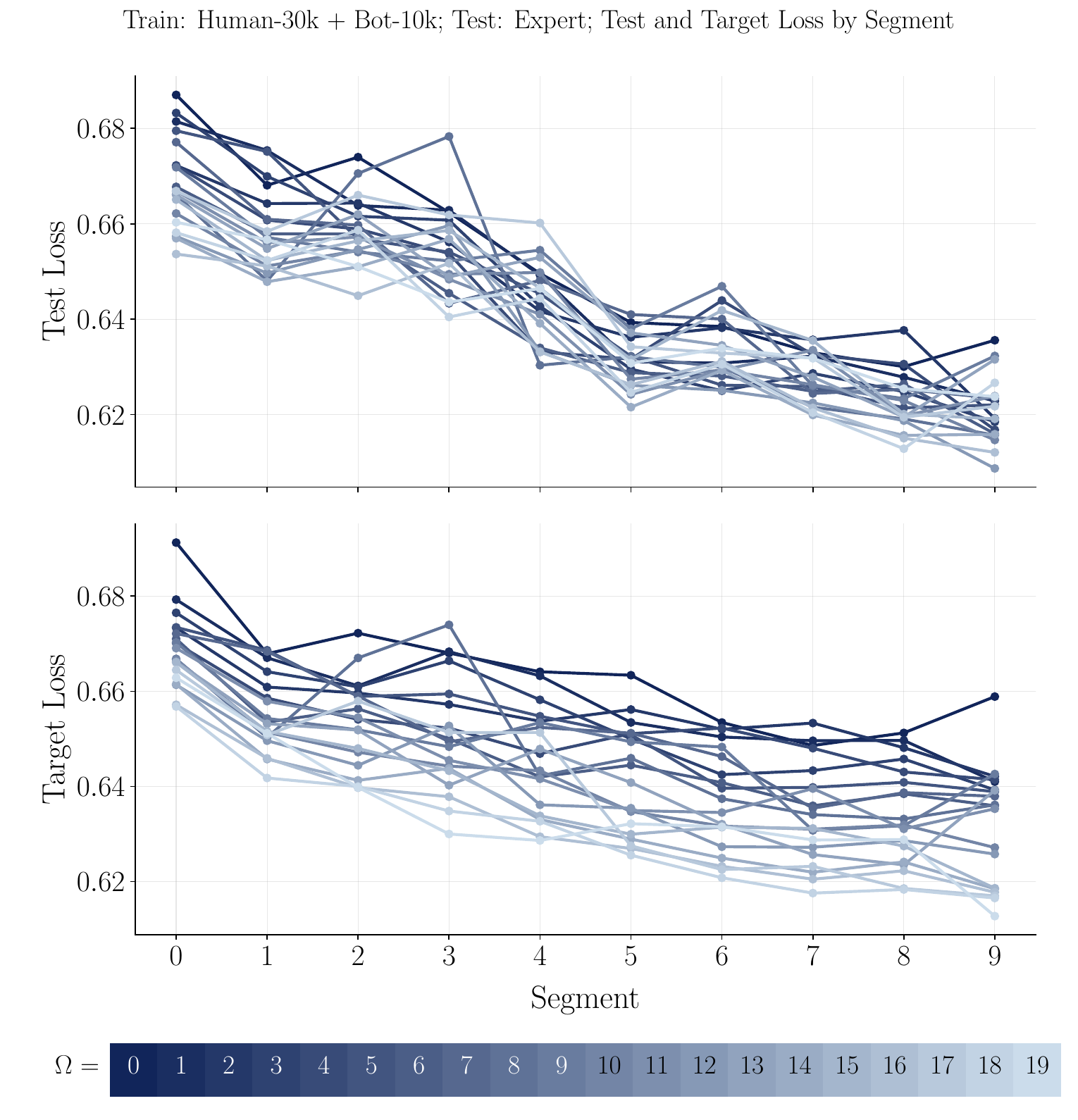}
    \caption{Test and target loss by segment for model trained on \texttt{Human-30k} + \texttt{Bot-10k}, targeting \texttt{Expert}.}
    \label{fig:human-30k-bot-10k-expert}
\end{figure}

\begin{figure}[H]
    \centering
    \includegraphics[width=0.54\linewidth]{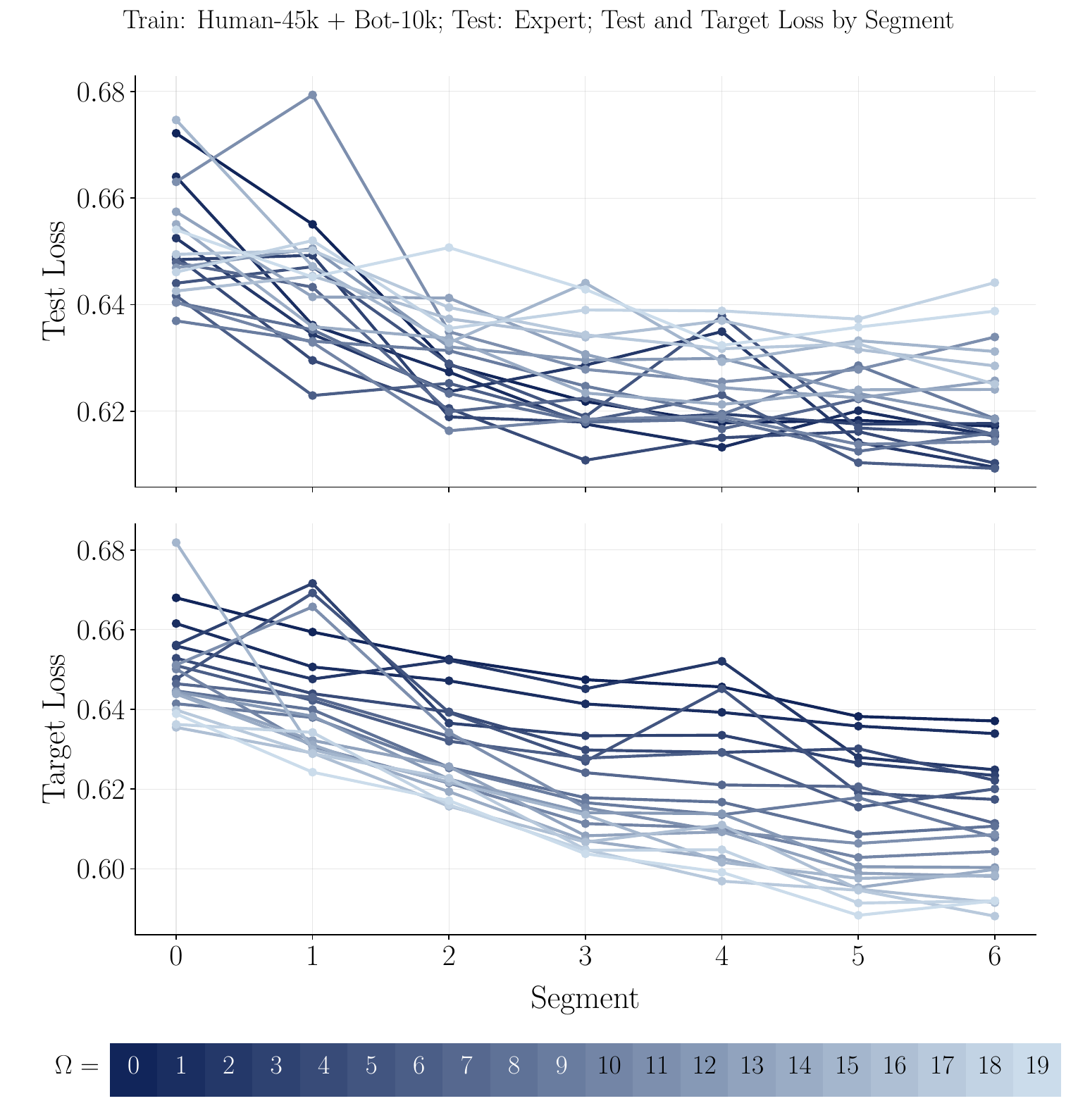}
    \caption{Test and target loss by segment for model trained on \texttt{Human-45k} + \texttt{Bot-10k}, targeting \texttt{Expert}.}
    \label{fig:human-45k-bot-10k-expert}
\end{figure}

\begin{figure}[H]
    \centering
    \includegraphics[width=0.54\linewidth]{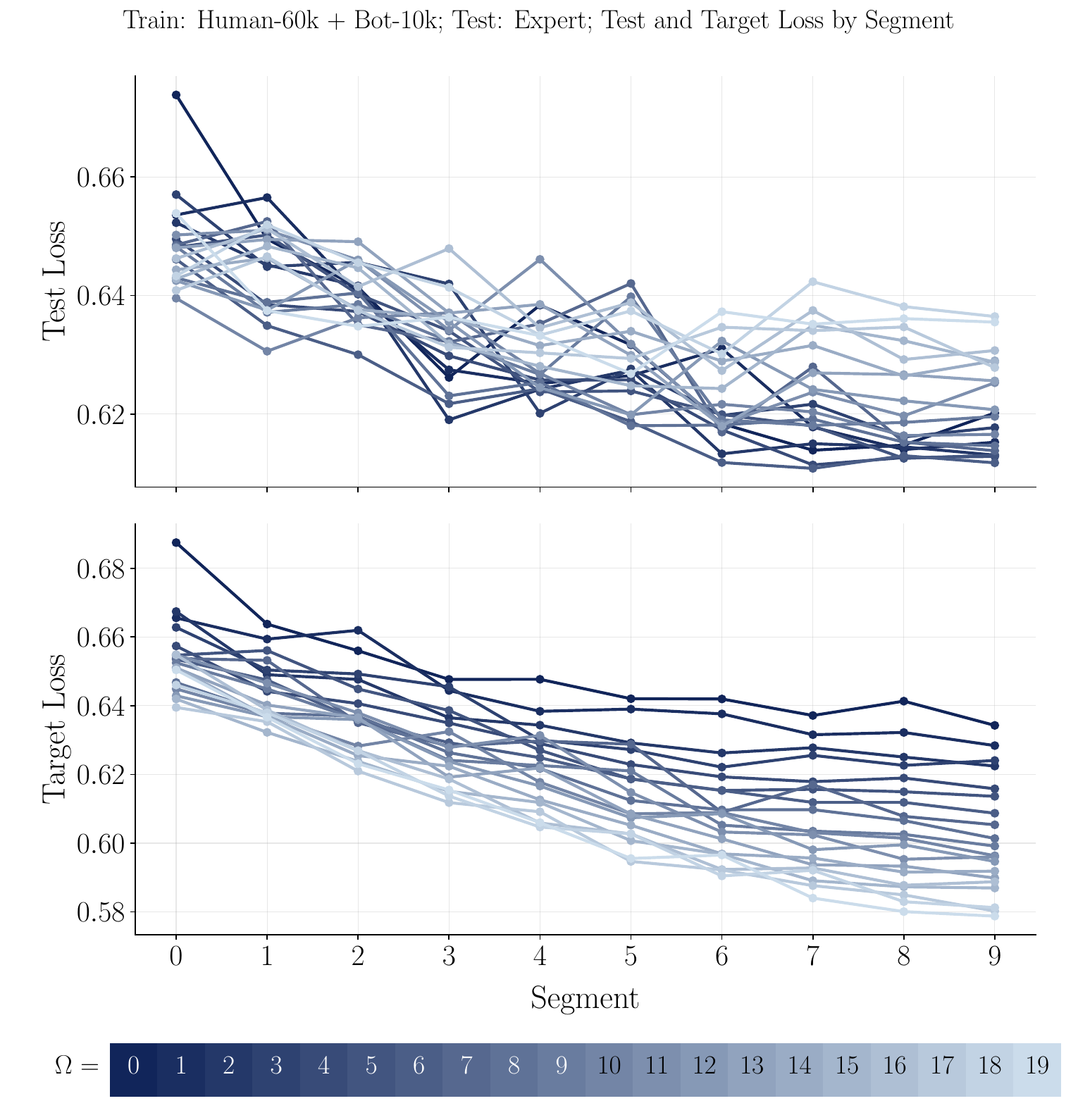}
    \caption{Test and target loss by segment for model trained on \texttt{Human-60k} + \texttt{Bot-10k}, targeting \texttt{Expert}. Notable overfitting to target observed.}
    \label{fig:human-60k-bot-10k-expert}
\end{figure}

\begin{figure}[H]
    \centering
    \includegraphics[width=0.54\linewidth]{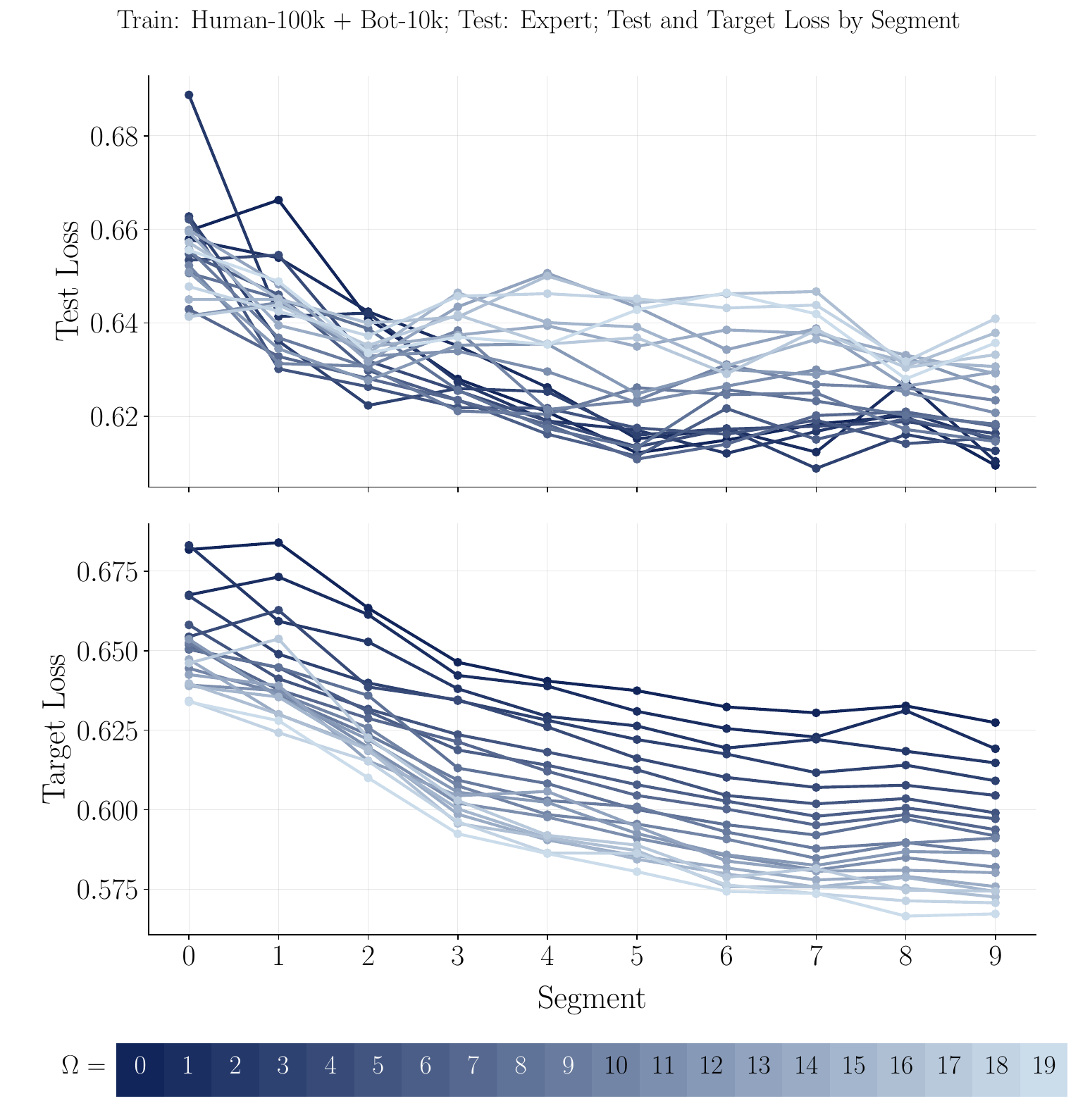}
    \caption{Test and target loss by segment for model trained on \texttt{Human-100k} + \texttt{Bot-10k}, targeting \texttt{Expert}. Notable overfitting to target observed.}
    \label{fig:human-100k-bot-10k-expert}
\end{figure}

\Needspace{0.42\textheight}
\subsection{Customizing to specific domains}
\label{app:customization}

\begin{figure}[H]
    \centering
    \includegraphics[width=0.66\linewidth]{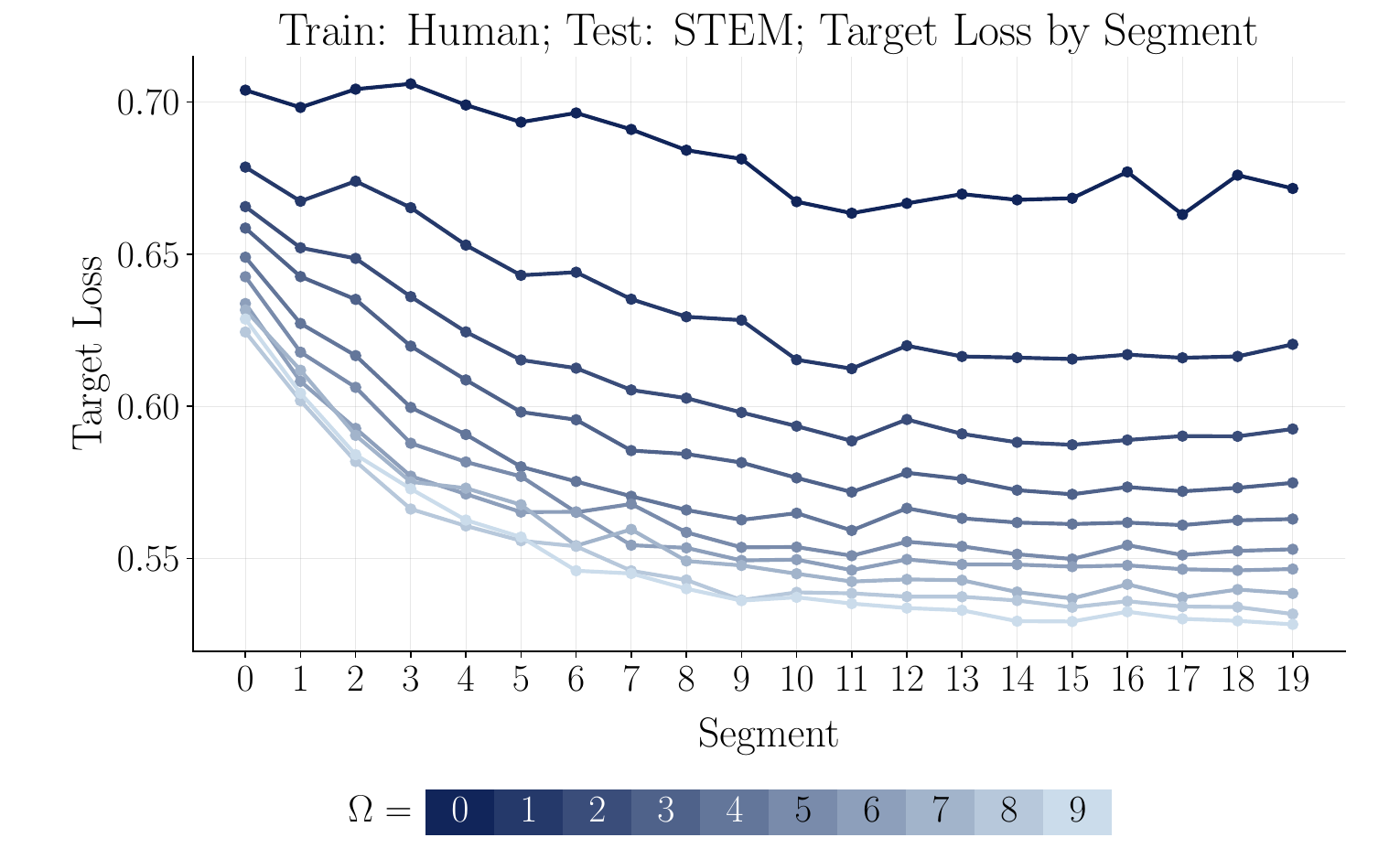}
    \caption{\textcolor{black}{Target loss} by segment for model trained on \texttt{Human-155k}, targeting \texttt{STEM}.}
    \label{fig:human-155k-stem}
\end{figure}

\Needspace{0.52\textheight}
\subsection{Incorporating synthetic data}
\label{app:synthetic-data}

\textcolor{black}{The human-only training trajectories
appear in Appendix~\ref{app:heterogeneous-data}. The following figures
show the synthetic-only and mixed-data trajectories and an equal-size
comparison.}

\begin{figure}[H]
    \centering
    \includegraphics[width=0.64\linewidth]{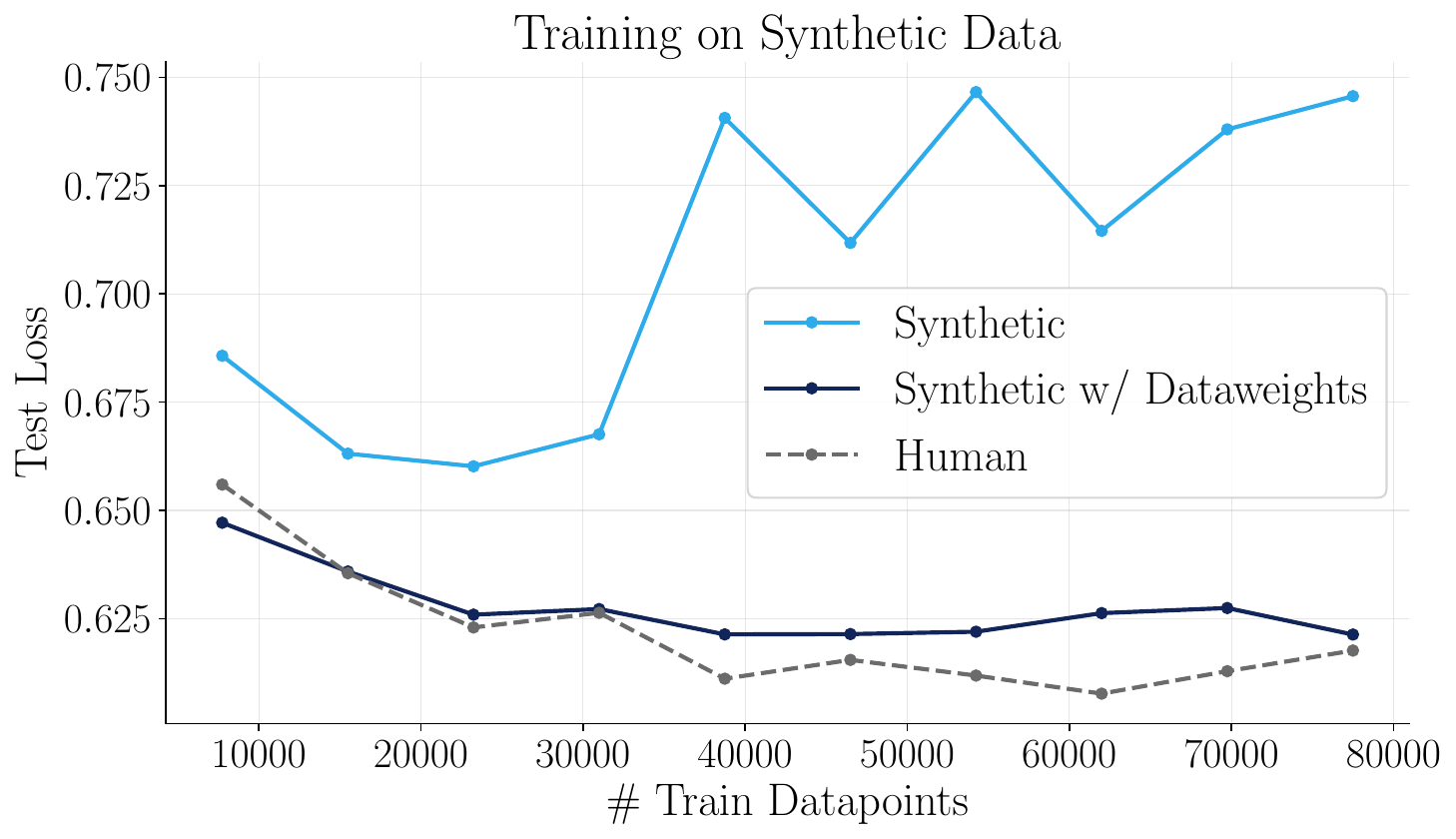}
    \caption{Test loss curves on \texttt{Expert} domain, while training on pure synthetic data (\texttt{Gemini-77k}), with and without dataweights. Test loss from training on \texttt{Human-77k} is shown for reference.}
    \label{fig:synth-data-curves}
\end{figure}

\begin{figure}[H]
    \centering
    \includegraphics[width=0.54\linewidth]{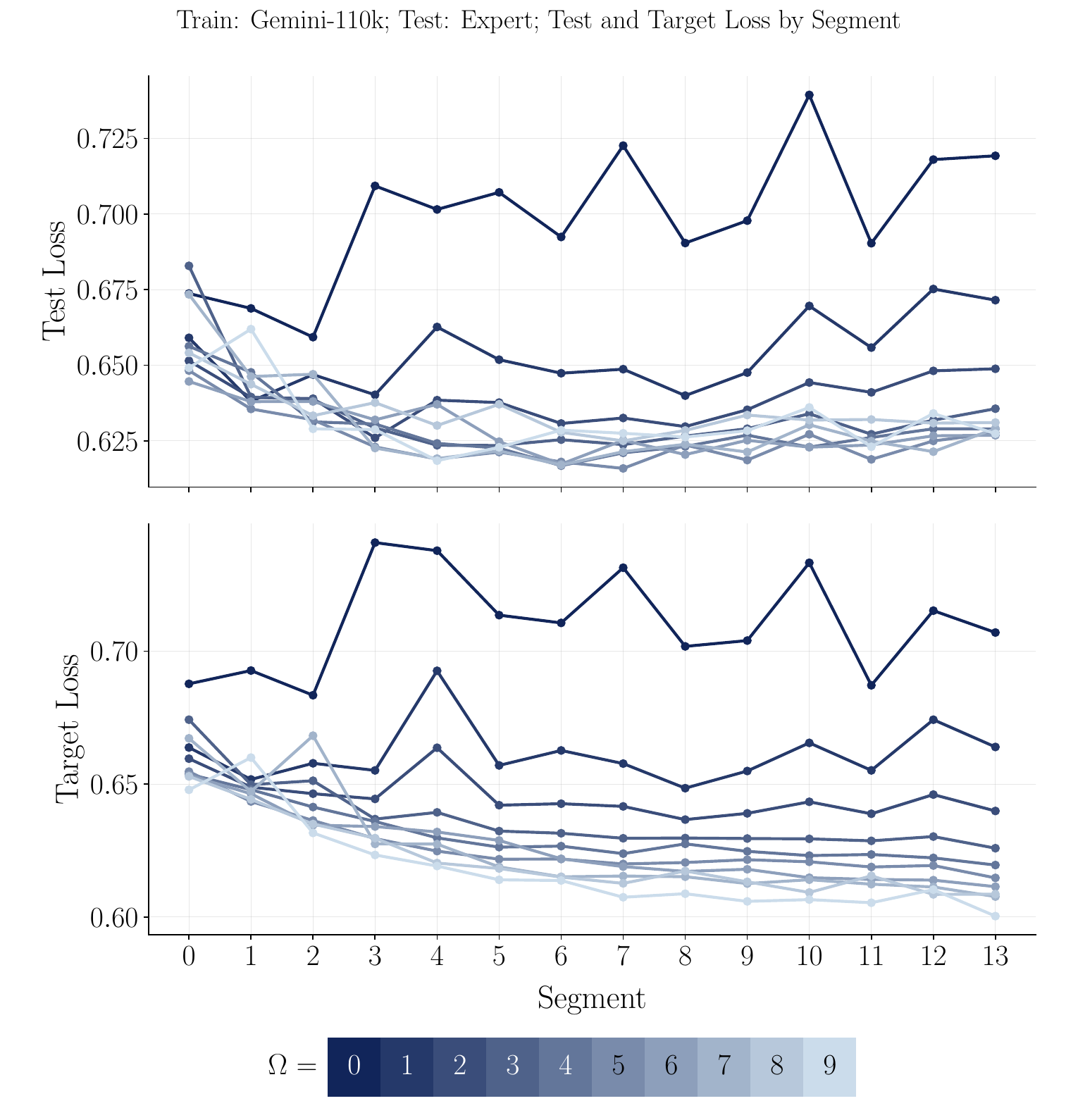}
    \caption{Test and target loss by segment for model trained on \texttt{Gemini-110k}, targeting \texttt{Expert}.}
    \label{fig:gemini-110k-expert}
\end{figure}

\begin{figure}[H]
    \centering
    \includegraphics[width=0.54\linewidth]{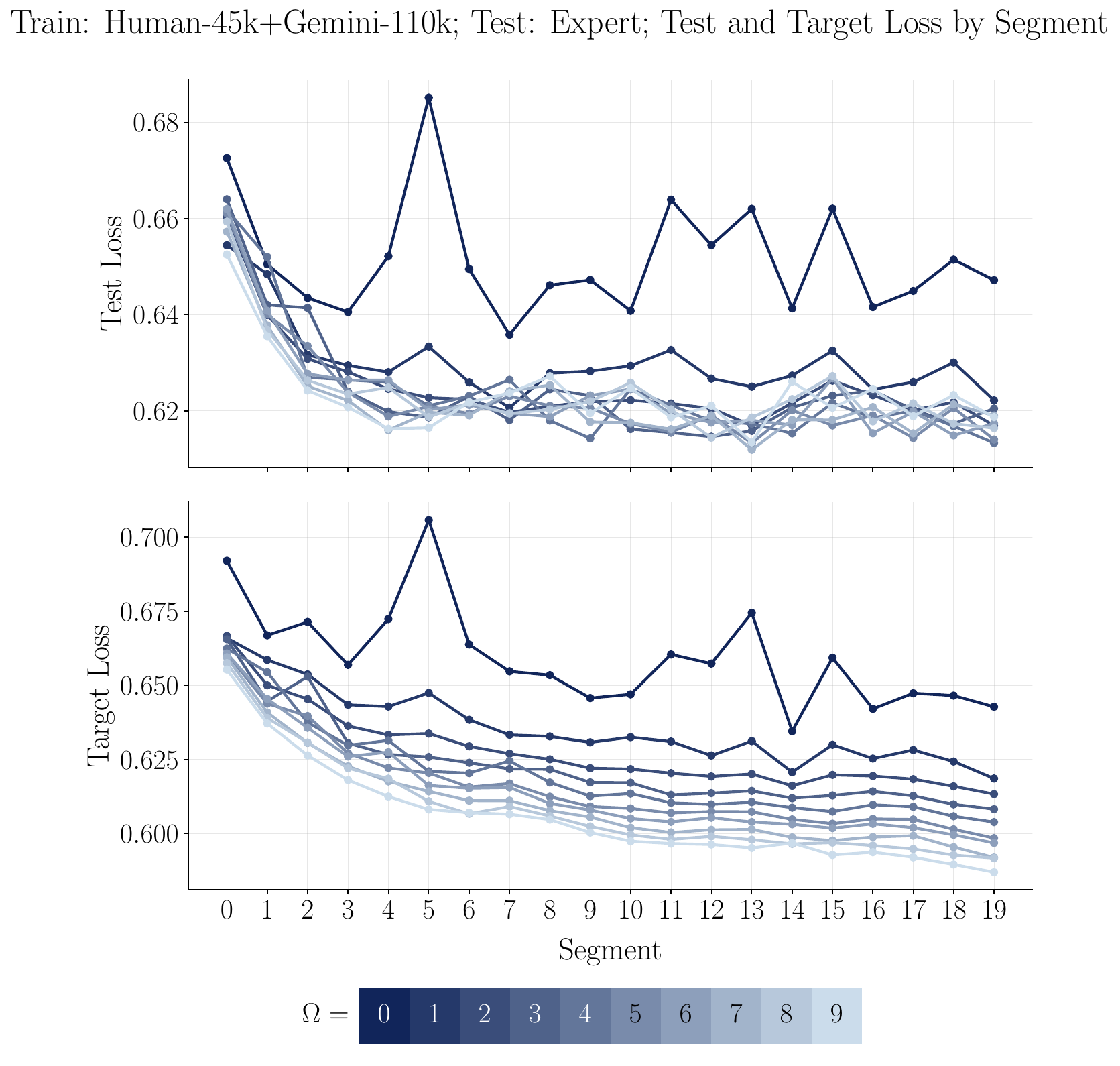}
    \caption{Test and target loss for model trained on \texttt{Human-45k} + \texttt{Gemini-110k}, targeting \texttt{Expert}.}
    \label{fig:human-45k-gemini-110k-expert}
\end{figure}

In particular, we notice test loss does not decrease significantly when dataweights are uniform. As metagradient steps are taken and dataweights are learned, test loss begins decreasing. In this example, the target and test curves are not perfectly in step, suggesting some overfitting may be occurring.

\begin{figure}[H]
    \centering
    \includegraphics[width=0.64\linewidth]{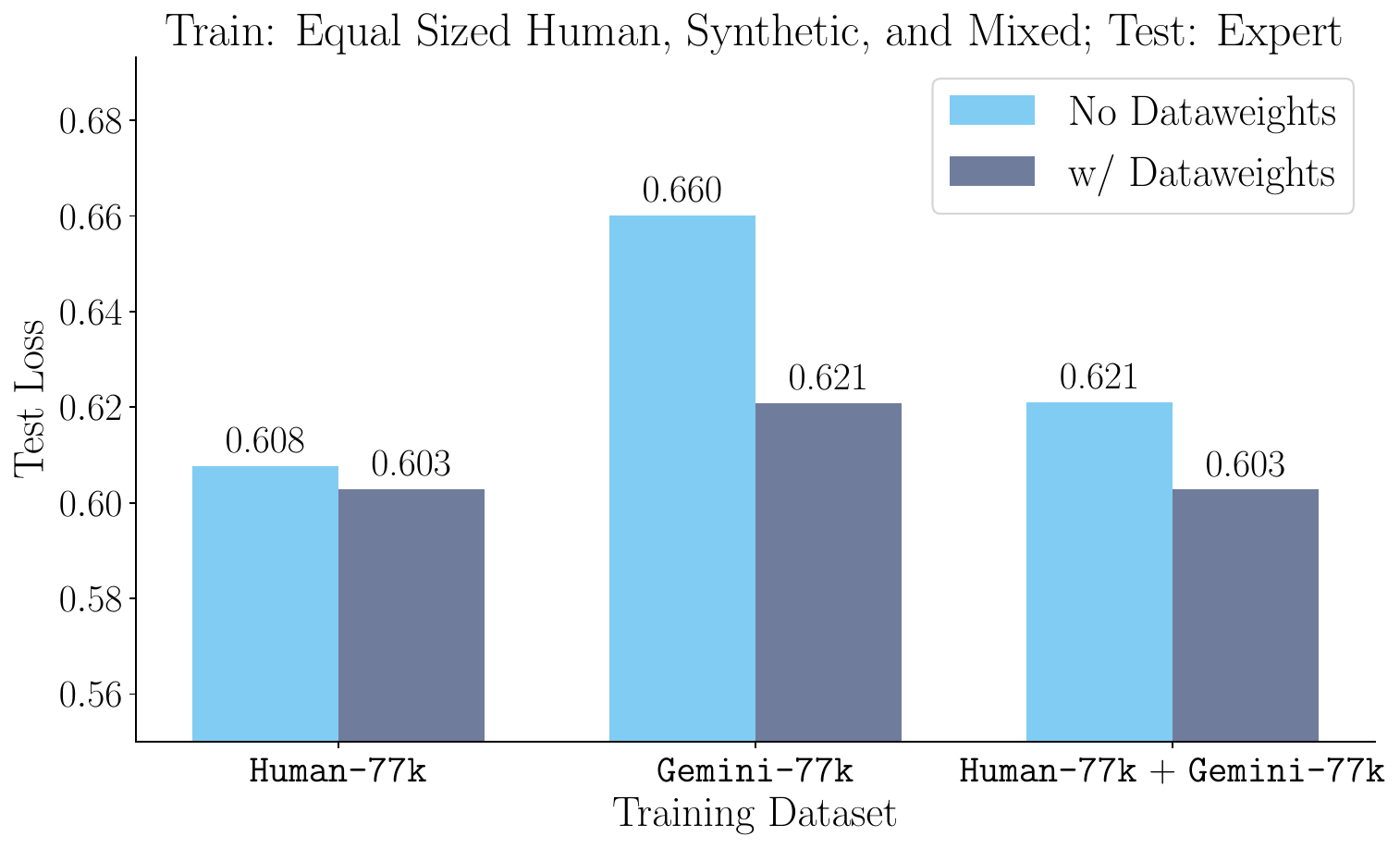}
    \caption{Test loss comparison for equal-sized \texttt{Human}, \texttt{Synthetic}, and mixed training sets, targeting \texttt{Expert}.}
    \label{fig:equal-size-comparison-expert}
\end{figure}

\begin{figure}[H]
    \centering
    \includegraphics[width=0.54\linewidth]{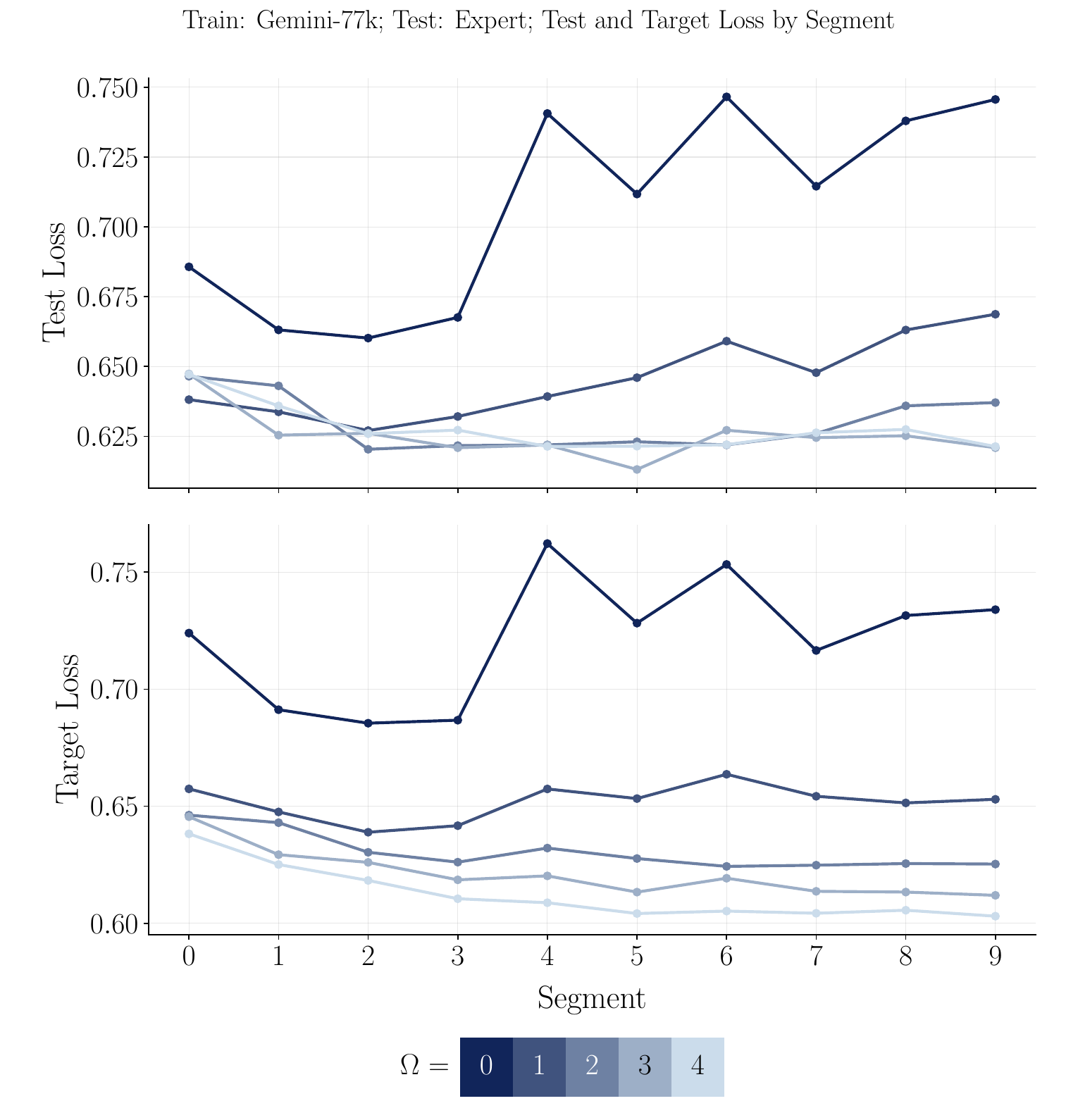}
    \caption{Test and target loss by segment for model trained on \texttt{Gemini-77k}, targeting \texttt{Expert}.}
    \label{fig:gemini-77k-expert}
\end{figure}

\begin{figure}[H]
    \centering
    \includegraphics[width=0.54\linewidth]{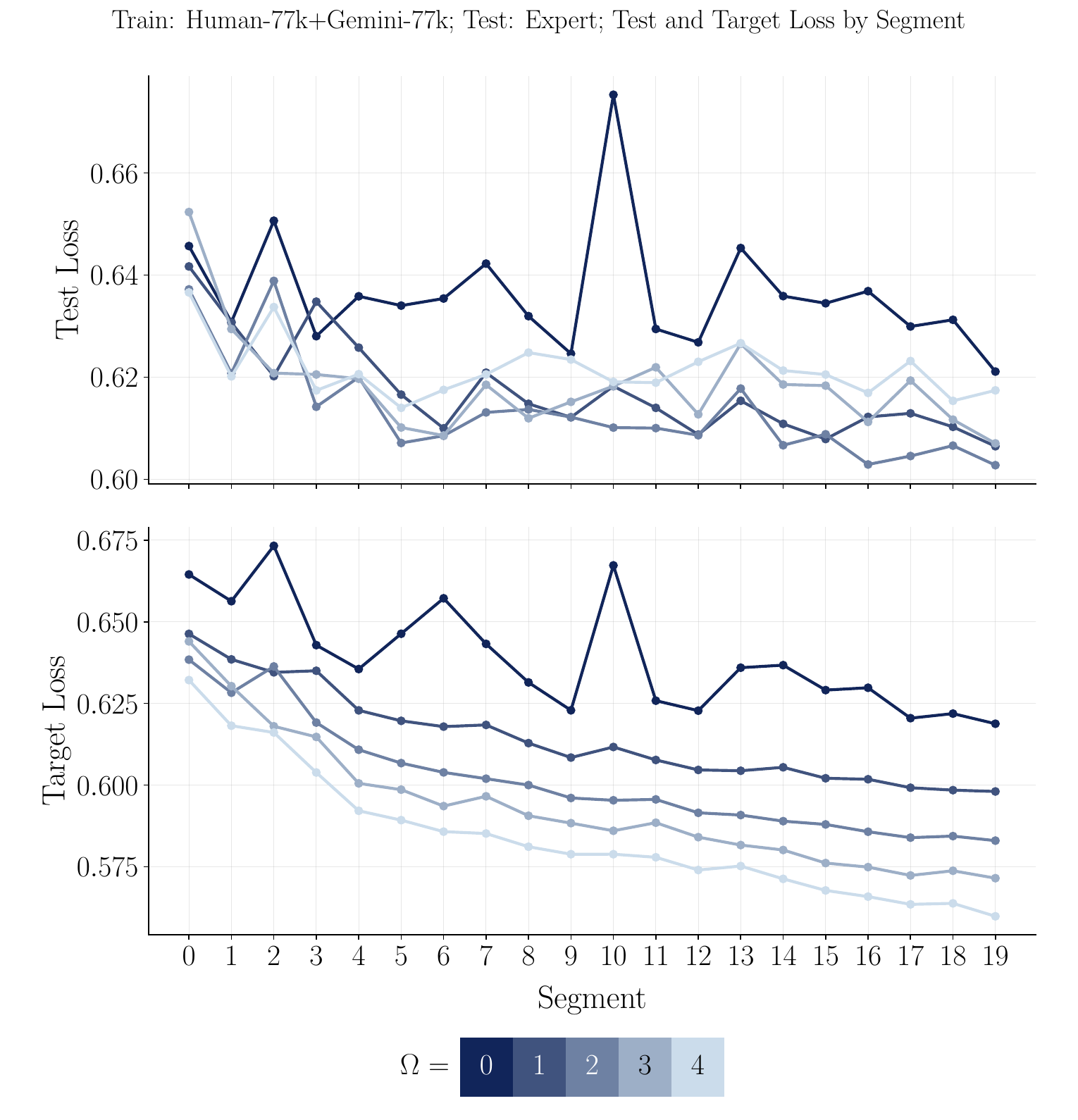}
    \caption{Test and target loss by segment for model trained on \texttt{Human-77k} + \texttt{Gemini-77k}, targeting \texttt{Expert}.}
    \label{fig:human-77k-gemini-77k-expert}
\end{figure}

\Needspace{0.52\textheight}
\subsection{Transferring dataweights to larger models}
\label{app:model-transfer}

\begin{figure}[H]
    \centering
    \includegraphics[width=0.54\linewidth]{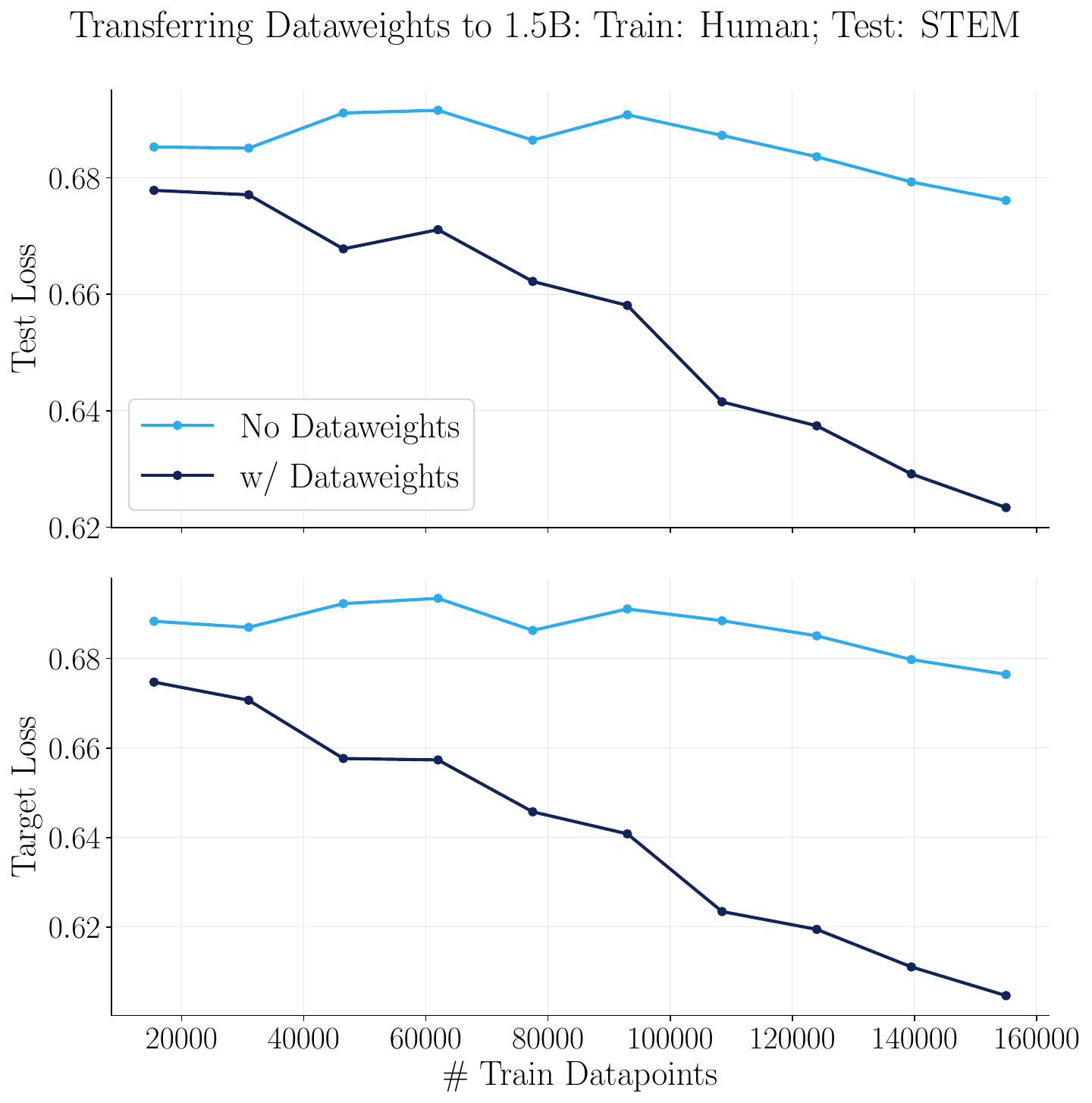}
    \caption{{\em Per-sample dataweights learned at 0.5B transfer to a 1.5B reward model.} Target and test loss curves on the \texttt{STEM} domain for a 1.5B reward model, with and without dataweights learned at the 0.5B scale.}
    \label{fig:transfer-stem}
\end{figure}

\begin{figure}[H]
    \centering
    \includegraphics[width=0.54\linewidth]{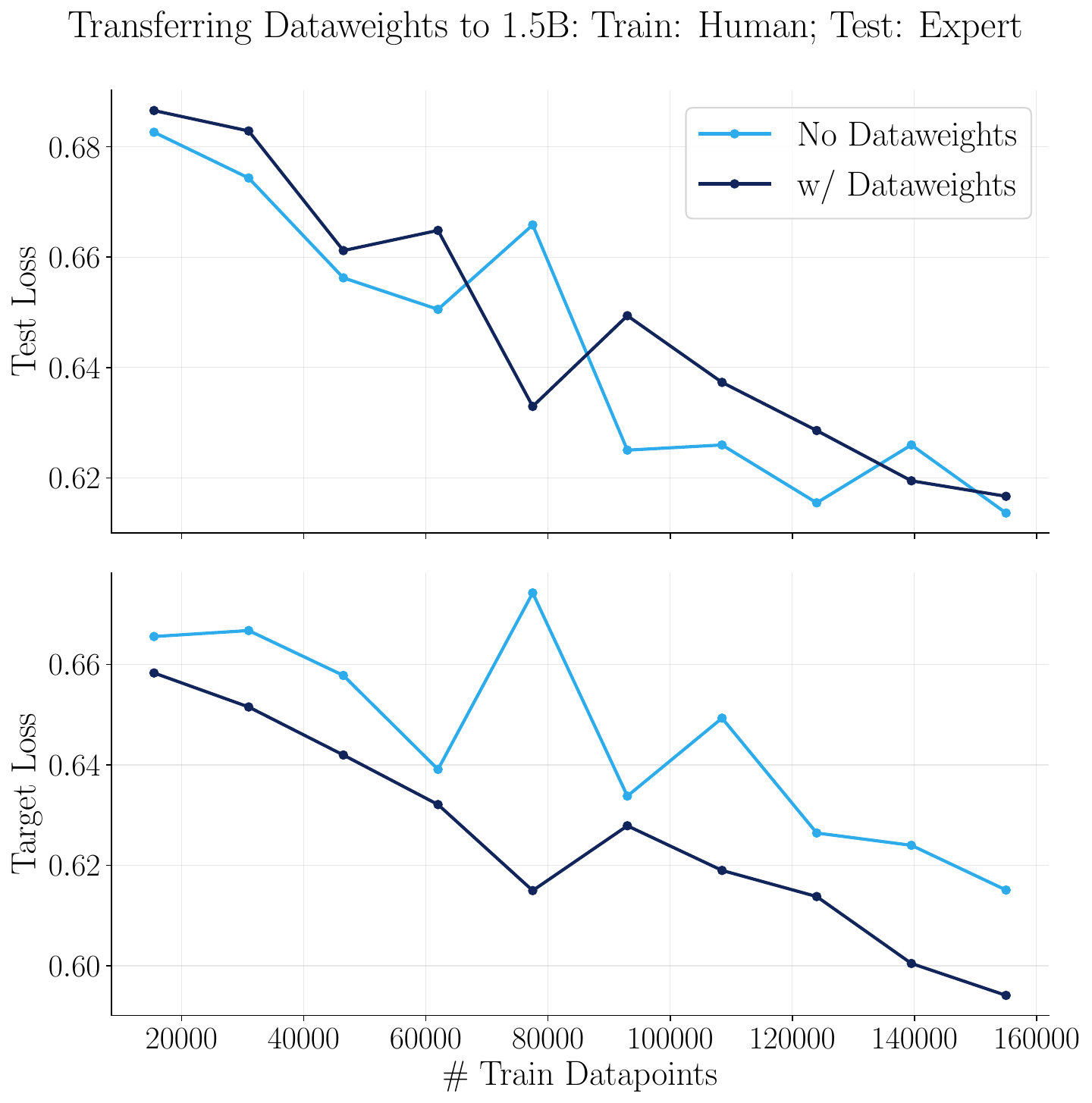}
    \caption{Transfer of \textcolor{black}{learned dataweights} to \textcolor{black}{a 1.5B-parameter reward model} trained on \texttt{Human} data, targeting \texttt{Expert}.}
    \label{fig:transfer-1.5b-human-expert}
\end{figure}

\begin{figure}[H]
    \centering
    \includegraphics[width=0.54\linewidth]{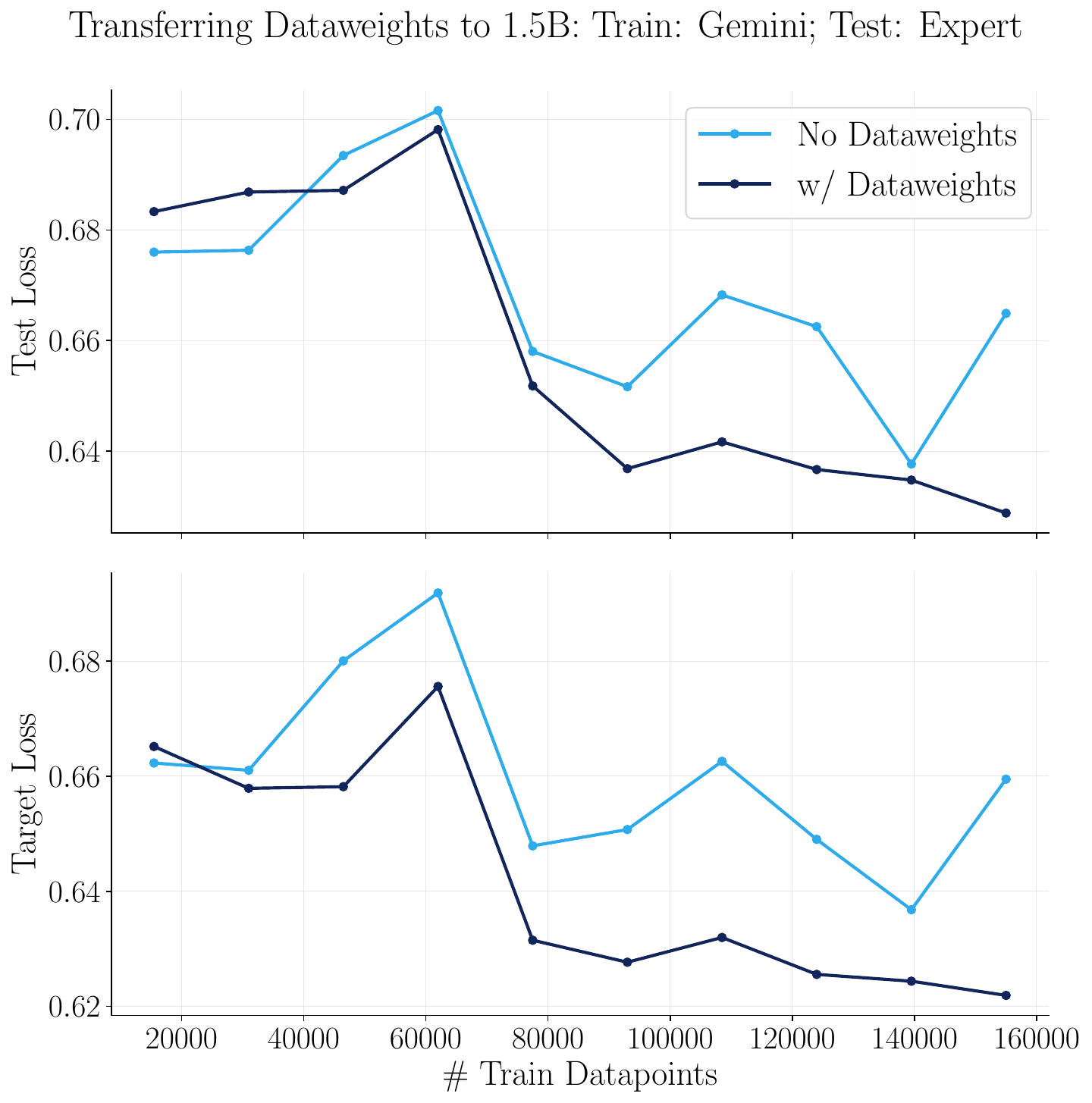}
    \caption{Transfer of \textcolor{black}{learned dataweights} to \textcolor{black}{a 1.5B-parameter reward model} trained on \texttt{Gemini} data, targeting \texttt{Expert}.}
    \label{fig:transfer-1.5b-gemini-expert}
\end{figure}

\clearpage
\section{Selection-baseline comparison: protocol and additional results}
\label{app:selection-baselines}

This appendix details the protocol behind Section~\ref{sec:selection-baselines} and
reports additional pool comparisons. \textcolor{black}{Figure~\ref{fig:selection-cost-benefit}
in the main text compares the methods' cost--performance tradeoffs.
We give the compute-cost accounting and additional two-pool trajectories below.}

\Needspace{8\baselineskip}
\subsection{Evaluation protocol and scorer configurations}
\label{app:selection-protocol}

Every scorer produces one score per training example (higher =
keep) for the same pool/target pair. For each keep fraction
$k \in \{10, 20, \ldots, 90, 95, 99\}$ we retain the top-scoring $k\%$ and tile the kept
rows back to the original pool size, so every arm performs an identical number of
optimizer steps regardless of $k$ (selection is therefore not confounded with a reduced
compute budget). Each arm trains with the same reward-model training procedure as all other experiments,
split into ten segments; we select the segment with the
lowest \textcolor{black}{loss on the checkpoint selection set} and report the loss on the disjoint held-out test split, so
\textcolor{black}{checkpoint selection does not use the held-out test set}. Every keep fraction has a
random-drop control, and the unweighted baseline is replicated over three seeds to
establish the noise band. Target and test splits were verified disjoint from each other
and from the training pools. Individual keep-fraction points are single-seed.

\paragraph{Scorer configurations.} \emph{TRAK:} per-example gradients are projected to
$1{,}024$ dimensions with a sparse Johnson--Lindenstrauss transform; scores follow the
primal form \textcolor{black}{$q_i^\top (Q^\top Q + \lambda I)^{-1} q_{\mathrm{tgt}}$},
z-scored per (model, checkpoint, projection) and averaged.
\textcolor{black}{Here, $q_i$ is the projected gradient for training example $i$,
$Q$ has rows $q_i^\top$, $q_{\mathrm{tgt}}$ is the projected target-loss gradient,
and $\lambda$ is the regularization parameter.}

\subsection{Cost--performance tradeoffs}
\label{app:selection-cost-performance}
Figure~\ref{fig:selection-cost-benefit-full} shows the full
dataweight-optimization trajectory.
Selection costs retain the original scoring-plus-one-retrain estimates, excluding
source-checkpoint generation and the keep-fraction search. 
Under this accounting, dataweighting first reaches the observed Pareto frontier at
$\Omega=3$ on \texttt{Human-77k} ($58.3$ GPU-hours, test loss $0.635$) and at
$\Omega=4$ on \texttt{Human-30k}+\texttt{Bot-10k} ($48.8$ GPU-hours, $0.637$).

\begin{figure}[H]
    \centering
    \includegraphics[width=0.84\linewidth]{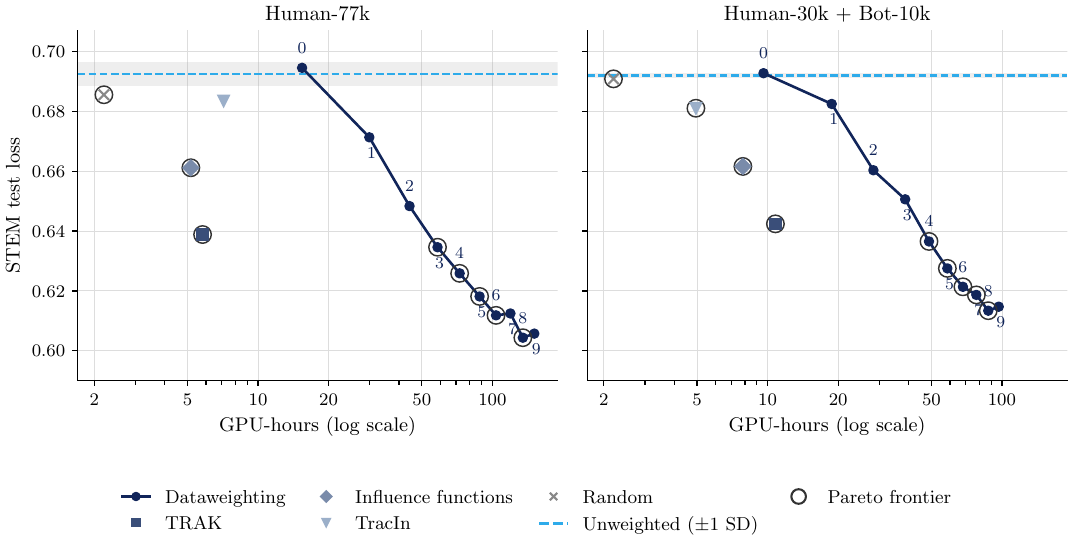}
    \caption{{\em Dataweighting's cost--performance trajectory on the \texttt{STEM}
    target.} Held-out test loss versus compute for two training pools; lower and left
    is better. Each dataweighting point is one completed meta-epoch, labelled by its
    zero-based index $\Omega=0,\ldots,9$, including the continuation run. At each
    budget, we report the test loss of the checkpoint with the lowest \textcolor{black}{selection-set loss}
    available so far. Baselines are shown at their best keep fraction, using the
    original scoring-plus-one-retrain cost estimates. Black rings mark the
    non-dominated points among the plotted observations. The grey band is $\pm 1$
    standard deviation over three unweighted seeds. Dataweighting requires no
    separate retraining stage.}
    \label{fig:selection-cost-benefit-full}
\end{figure}

\Needspace{0.60\textheight}
\subsection{Response-length audit}
\label{app:selection-length-audit}

We inspect the retained \texttt{Human-155k} pairs before repetition, using the
preference labels to identify the chosen and rejected responses. Length is the
number of characters across response turns before tokenization or truncation;
responses are identified by turn position, since some \texttt{STEM} responses
are stored with a user-role tag rather than an assistant-role tag.
Equal-length pairs remain in the denominator.
At $10\%$ retention, the fraction preferring the shorter response is
$80.6\%$ for TracIn, $57.2\%$ for influence functions, and $45.3\%$ for TRAK,
compared with $39.9\%$ for random selection and $40.3\%$ in the full pool.
Whitespace-delimited word counts give the same pattern
($80.1\%$, $56.6\%$, $44.3\%$, $39.8\%$, and $40.1\%$, respectively).
At $90\%$ retention, the character-based fractions return closer to the pool
distribution: $41.7\%$, $41.3\%$, and $40.1\%$ for TracIn, influence functions,
and TRAK. Their test losses also decrease from $1.739$, $1.764$, and $1.023$
at $10\%$ retention to $0.643$, $0.627$, and $0.619$ at $90\%$.
This descriptive audit identifies a selection-induced distribution shift,
not whether the original preferences were caused by length or whether length skew
causes the generalization degradation.

\begin{figure}[H]
    \centering
    \includegraphics[width=0.94\linewidth]{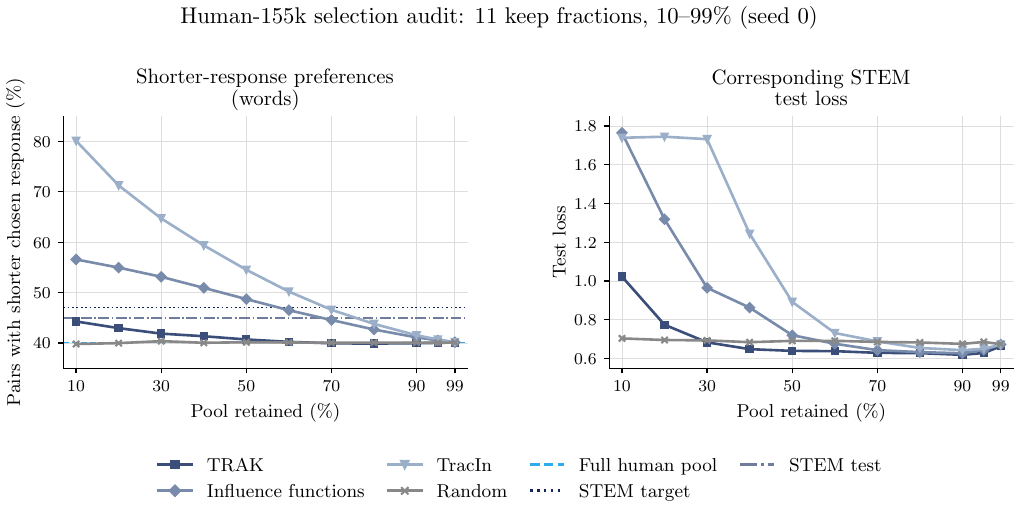}
    \caption{{\em Aggressive selection skews response-length preferences.}
    Audit of \texttt{Human-155k} subsets retained by TRAK, influence functions,
    TracIn, and random selection across eleven keep fractions (one seed).
    Left and middle: fraction of retained pairs whose chosen response is shorter
    than the rejected response, measured in characters and words, respectively;
    equal-length pairs remain in the denominator. Horizontal lines show the
    full human pool and the \texttt{STEM} target and test distributions.
    Right: corresponding test loss at the \textcolor{black}{checkpoint selected using the selection set}.
    TracIn and influence functions increasingly favor shorter-response preferences
    at low retention. These curves describe the selected data, rather than directly
    measuring the trained models' reliance on response length.}
    \label{fig:selection-length-bias}
\end{figure}

\Needspace{0.80\textheight}
\subsection{Comparisons across training pools}
\label{app:selection-pools}

\begin{figure}[H]
    \centering
    \includegraphics[width=0.72\linewidth]{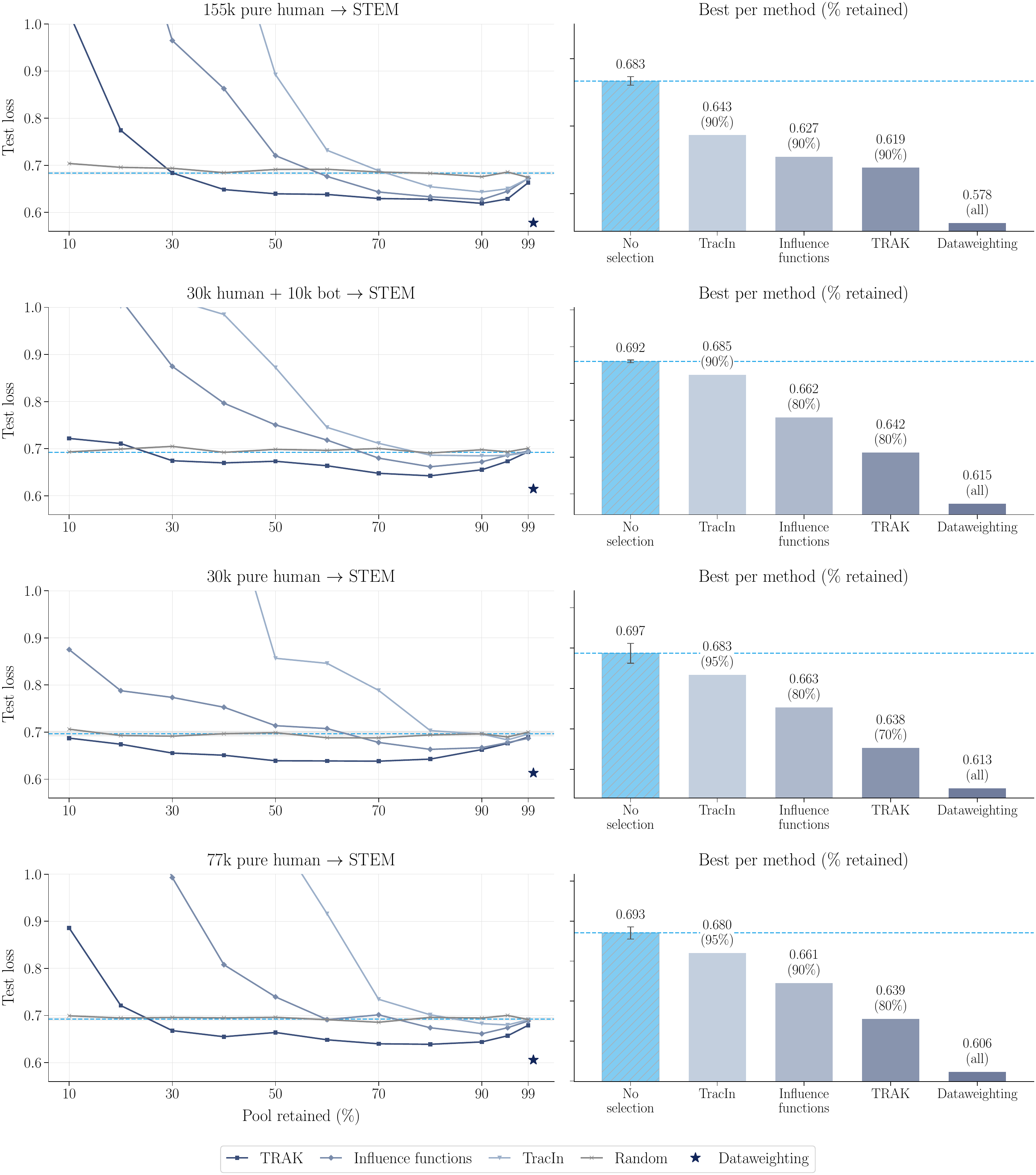}
    \caption{{\em The main-paper comparison, on every pool.}
    \textcolor{black}{Figure~\ref{fig:selection-baselines} repeated for all four training pools against the}
    \texttt{STEM} target, same protocol and same panels throughout. Rows 2--3 are a
    \emph{paired} pool: \texttt{Human-30k} is the poisoned pool with its $10{,}000$ bot
    rows removed and nothing else changed, so reading them adjacently separates what the
    contamination costs each scorer from what pool size costs it. \textcolor{black}{Dataweighting attains
    the lowest test loss on all four pools, followed by TRAK and influence functions.} Dataweighting's margin over the strongest baseline is
    $0.041 / 0.028 / 0.025 / 0.033$ (155k, poisoned 40k, clean 30k, 77k).}
    \label{fig:selection-all-pools}
\end{figure}

\clearpage
\section{Experimental setup}
\label{app:run-hyperparameters}
\providecommand{\hpsetupfig}[2]{\ifcsname r@#1\endcsname Fig.~\ref{#1}\else #2\fi}

Tables~\ref{tab:hp-inner}--\ref{tab:hp-parametric} summarize reward-model training,
dataweight optimization, and the parametric weighting model.
One training pass is one epoch over the pool. $K$ denotes segments per pass,
$O$ the configured number of outer passes, and $I$ the inner repeats per segment
(\texttt{inner\_repeats}). Each inner repeat replays the segment from its starting
model state, computes a metagradient, and updates the dataweights. Thus $I=1$
means one meta-optimization traversal per segment, not one model-optimizer step;
each traversal can contain multiple optimizer steps. Forward-only training uses
$I=0$. Autodiff recomputation and optional extra forward passes are not counted in $I$.

\begin{table*}[!t]
\centering
\caption{Reward-model training hyperparameters (inner optimization).}
\label{tab:hp-inner}
\begingroup\small\renewcommand{\arraystretch}{1.12}
\begin{tabularx}{\textwidth}{lX}
\toprule
Hyperparameter & Setting \\
\midrule
Base model & Qwen2.5-Instruct, 0.5B by default; larger models for transfer \\
Reward head & Linear scalar head on the \texttt{<|cls|>} hidden state \\
Training objective & Bradley--Terry loss \\
Optimizer & AdamW \\
Learning rate & $2\times10^{-6}$ \\
Batch size & 128 preference pairs by default \\
Epochs per inner-training pass & 1 \\
Mixed-precision training & FP32/BF16 \\
Maximum context length & 4,096 tokens \\
\bottomrule
\end{tabularx}
\endgroup
\end{table*}

\begin{table*}[!t]
\centering
\caption{Dataweight optimization settings and the corresponding figures.}
\label{tab:hp-outer}
\begingroup\small\setlength{\tabcolsep}{4pt}\renewcommand{\arraystretch}{1.15}
\begin{tabularx}{\textwidth}{>{\raggedright\arraybackslash}X>{\raggedright\arraybackslash}Xlcccc}
\toprule
Setup & Figure & Optimizer & LR & $K$ & $O$ & $I$ \\
\midrule
STEM, per-sample & \hpsetupfig{fig:transfer-bars}{Parametric comparison}(a) & AdamW & 0.1 & 10 & 10 & 1 \\
STEM, linear source & \hpsetupfig{fig:transfer-bars}{Parametric comparison}(a--c) & AdamW & 0.1 & 10 & 10 & 1 \\
Ranking, linear source & \hpsetupfig{fig:ranking-transfer-by-size}{Ranking transfer} & AdamW & 0.1 & 10 & 10 & 1 \\
Ranking, per-sample & \hpsetupfig{fig:ranking-distribution-splits}{Ranking distribution} & SignSGD & 0.1 & 10 & 10 & 1 \\
Selection, per-sample & \hpsetupfig{fig:selection-baselines}{Selection comparison}; \hpsetupfig{fig:selection-all-pools}{all pools} & SignSGD & 0.1 & 10 & 5--10 & 1 \\
Human-30k, full-pass & \hpsetupfig{fig:selection-cost-benefit}{Cost--performance} & SignSGD & 0.1 & 1 & 15 & 1 \\
Human-30k, single-step & \hpsetupfig{fig:selection-cost-benefit}{Cost--performance} & SignSGD & 0.1 & 234 & 15 & 1 \\
Forward-only transfer & \hpsetupfig{fig:transfer-bars}{Parametric comparison}(b,c); \hpsetupfig{fig:ranking-transfer-by-size}{ranking transfer} & --- & --- & 1 & 1 & 0 \\
\bottomrule
\end{tabularx}
\endgroup
\par\smallskip\begin{minipage}{\textwidth}\footnotesize
``Source'' denotes learning dataweights at 0.5B; destination models use fixed
weights. Selection budgets refer to individual original/continuation configurations.
\end{minipage}
\end{table*}

\begin{table*}[!t]
\centering
\caption{Parametric linear dataweight model used for
\hpsetupfig{fig:transfer-bars}{the parametric comparison} and
\hpsetupfig{fig:ranking-transfer-by-size}{ranking transfer}.}
\label{tab:hp-parametric}
\begingroup\small\renewcommand{\arraystretch}{1.12}
\begin{tabularx}{\textwidth}{lX}
\toprule
Component & Setting \\
\midrule
Embedding model & OpenAI \texttt{text-embedding-3-small}, frozen \\
Input texts & Prompt and the two candidate responses, embedded separately \\
Embedding dimensions & $3\times1{,}536=4{,}608$ per preference pair \\
Feature normalization & $\ell_2$ normalization of the concatenated feature vector \\
Weighting head & Single affine map $s_i=\mathbf{a}^{\top}\phi(z_i)+b$; no hidden layers \\
Trainable parameters & 4,609: 4,608 coefficients and one bias \\
Transfer & Freeze the learned head and apply it to destination-pool embeddings \\
\bottomrule
\end{tabularx}
\endgroup
\end{table*}

\endgroup

\clearpage
{\color{black}
\section{Stylized theoretical results for dataweighting}
\label{app:theory}
This appendix gives the formal results underlying Section~\ref{sec:theory}.
The constructions isolate three benefits of target-guided dataweighting:
correcting systematic synthetic-label bias, filtering severe corruption, and
adapting the training distribution to a specialized target. They concern global
optimization over fixed-size subsets, a restricted family of dataweights; they
do not assert that a particular metagradient algorithm finds the optimum.

\subsection{Shared model and learning procedures}
\label{app:theory-setup}

{\color{black}
We use overparameterized linear regression, where fitting a small target set
can leave prediction error on unseen examples. Inputs have two signal features
and $D$ weak nuisance features, so the full dimension is $d=D+2$:
\par}

\begin{equation}
    x=\bigl(g,\sqrt{\lambda/D}\,h\bigr),
    \qquad h\sim\mathcal{N}(0,I_D),\qquad \lambda>0.
\end{equation}

{\color{black}
Source inputs have $g\sim\mathcal{N}(0,I_2)$; target inputs have
$g\sim\mathcal{N}(0,T)$, with $T\succ0$. In both domains, $g$ and $h$ are
independent. Thus $\Sigma_S$ and $\Sigma_T$ are block diagonal with signal
blocks $I_2$ and $T$, respectively, and a shared nuisance block
$(\lambda/D)I_D$. Throughout, define
\begin{equation}
\label{eq:theory-biased-teacher}
    f^*(x)=g^\top\theta^*,
    \qquad
    \widehat f(x)=f^*(x)+\delta\,\operatorname{sign}(g_2)\mathbf{1}\{h_1>0\},
    \qquad \delta\ge0.
\end{equation}
Each result specifies the parameter choices in this construction.
We observe $n$ source examples and $n'$ target examples, all independent, with
$D+2>n+n'$. Source labels are $\widehat f(x)$; target labels are $f^*(x)$.
Write their design matrices as $X,Z$ and label vectors as $y,t$.

\paragraph{\normalcolor Learning procedures.}
Fix $1\le k\le n$. For each $k$-element source subset $S$, fit the minimum-norm interpolator
$\beta_S=X_S^\dagger y_S$. Dataweighting ($\mathrm{DW}$) returns the fit minimizing
$(n')^{-1}\|Z\beta_S-t\|^2$, breaking ties uniformly, without further fine-tuning.
This corresponds to weights $w_i=\mathbf{1}\{i\in S\}$.
The baselines fit either all source examples ($\mathrm{full+FT}$) or a uniformly
random $k$-subset ($\mathrm{random+FT}$), then fine-tune all parameters on the
same target set by squared-loss gradient flow to convergence.
The main text denotes the minimum of these two expected risks at dimension
$d$ by $\mathcal{R}_{FT}(n,n',k,d)$.

\paragraph{\normalcolor Risk and limits.}
We write the full parameter as $\beta=(\theta,v)\in\mathbb{R}^d$, where
$\theta\in\mathbb{R}^2$ contains the signal coefficients and
$v\in\mathbb{R}^D$ the nuisance coefficients. Its target population risk is
\par}

\begin{equation}
\label{eq:theory-population-risk}
    R_D(\beta)
    =(\theta-\theta^*)^\top T(\theta-\theta^*)
      +\frac{\lambda}{D}\|v\|^2.
\end{equation}

{\color{black}
For each procedure $A$, define
\par}

\begin{equation}
\label{eq:theory-limit-risk}
    \mathcal{R}_A(n)=\limsup_{D\to\infty}\mathbb{E}R_D(\beta_A).
\end{equation}

{\color{black}
Expectations include sampling and tie-breaking. We take $D\to\infty$ before
$n\to\infty$, keeping $n'$ fixed; the filtering result instead fixes $n,n',D$.
Each theorem specifies $\lambda$ and $T$, which may depend on $n$ but remain
fixed as $D\to\infty$.
\par}

\subsection{Reduction to two dimensions and validation control}
\label{app:theory-reduction}

{\color{black}
Let $G$ and $V$ contain the source and target signal features.
Standard Gaussian Gram-matrix limits and minimum-norm regression identities
give the limiting source coefficients
\begin{equation}
\label{eq:theory-ridge-reduction}
    \theta_S=(G_S^\top G_S+\lambda I_2)^{-1}G_S^\top y_S.
\end{equation}
Thus nuisance features induce an effective ridge penalty $\lambda$.
Fine-tuning from limiting source coefficients $\theta_0$ gives
\par}
\nopagebreak
\begin{equation}
\label{eq:theory-ft-reduction}
    \theta_{\mathrm{FT}}-\theta^*
    =\lambda(V^\top V+\lambda I_2)^{-1}(\theta_0-\theta^*).
\end{equation}

{\color{black}
The nuisance contribution to risk vanishes, leaving predictions that depend
on only two coefficients, regardless of how many subsets we search.
This yields a useful selection guarantee. Write
$R_T(\theta)=(\theta-\theta^*)^\top T(\theta-\theta^*)$.
For $n'\ge4$ and any reference subset $S_0$ chosen using source data alone,
with a definition independent of $D$,
\par}

{\color{black}
\begin{equation}
\label{eq:theory-validation-bound}
    \mathcal{R}_{\mathrm{DW}}(n)
    \le C_{n'}\,\mathbb{E}R_T(\theta_{S_0}),
    \qquad C_{n'}=\frac{4(n'-1)}{n'-3}\le12.
\end{equation}
\par}

{\color{black}
The constant $C_{n'}$ is independent of $T$, so the bound applies at each $n$
when the target covariance varies with the source sample size.

\begin{proof}[Proof of the selection bound]
Write $V=UT^{1/2}$ with $U$ standard Gaussian and $W=U^\top U$.
Comparing empirical and population quadratic forms, target-loss minimization
gives $R_T(\theta_{\mathrm{DW}})\le\kappa(W)R_T(\theta_{S_0})$ in the limiting
problem. Here $\kappa(W)=\lambda_{\max}(W)/\lambda_{\min}(W)$, and standard
Wishart moments give
$\mathbb{E}\kappa(W)\le\mathbb{E}[\operatorname{tr}W\operatorname{tr}W^{-1}]=C_{n'}$.
Independence of target and source data proves the bound in the limiting problem.

At fixed $n$, convergence is joint over the finitely many subsets.
Labels depend on nuisance features only through $h_1$, so standard
inverse-Wishart moment bounds on the remaining independent features give
uniform integrability of the risks. Fine-tuning orthogonally projects the
coefficient error, preserving this control. The bound therefore also holds
in expectation as $D\to\infty$ for the finite-$D$ selector, even if its
minimizing subset changes with $D$.
\end{proof}
\par}

\subsection{Correcting small synthetic-label bias}
\label{app:theory-bias}

{\color{black}
Take $T=I_2$, $\theta^*=e_1$, and fixed $\lambda,\delta>0$ in
\eqref{eq:theory-biased-teacher}.
Half the source labels are correct; the rest have a small systematic error
correlated with the irrelevant feature $g_2$. Averaging more source data
preserves this bias, while a growing clean subset allows consistent learning.
\par}

\begin{theorem}[Bias correction]
\label{thm:theory-bias}
Fix $0<\alpha<1/2$, $n'\ge4$, and $\lambda,\delta>0$, and set
$k=\lfloor\alpha n\rfloor$. Then
{\color{black}
\begin{equation}
\label{eq:theory-bias-result}
\begin{aligned}
    \mathcal{R}_{\mathrm{DW}}(n)&\longrightarrow0,\\
    \liminf_{n\to\infty}\min\{\mathcal{R}_{\mathrm{full+FT}}(n),
        \mathcal{R}_{\mathrm{random+FT}}(n)\}
        &\ge c_{\lambda,n'}\delta^2>0,
\end{aligned}
\end{equation}
\par}
where $c_{\lambda,n'}=(2\pi)^{-1}[\lambda/(\lambda+n')]^2$ is valid.
The two baselines have the same limiting risk.
\end{theorem}

{\color{black}
\begin{proof}
Put $b=\delta/\sqrt{2\pi}$. Since
$\mathbb{E}[g(\widehat f-f^*)]=(0,b)^\top$, both full-data and random-subset
fits converge to $e_1+be_2$. The convergence holds in mean square by Gaussian
covariance concentration: the squared coefficient error is bounded on
well-conditioned designs and is $O(1+N)$ otherwise, for a fit on $N$ examples;
the latter event has exponentially small probability.
Rotational invariance and Jensen's inequality give the fine-tuning risk factor
\begin{equation}
    q_{\lambda,n'}
    =\frac12\mathbb{E}\operatorname{tr}
       \bigl[\lambda^2(V^\top V+\lambda I_2)^{-2}\bigr]
    \ge\left(\frac{\lambda}{\lambda+n'}\right)^2.
\end{equation}
Both baselines therefore have limiting risk
$q_{\lambda,n'}b^2\ge c_{\lambda,n'}\delta^2$.

For dataweighting, use the first $k$ clean examples as a comparator, falling
back to the first $k$ examples if necessary. Cleanliness depends only on
$h_1$, so the clean signal features remain independent standard Gaussians.
Their ridge bias has norm at most one and vanishes as $k\to\infty$ with fixed $\lambda$.
Hoeffding's inequality bounds the fallback probability by $e^{-c_\alpha n}$,
where $c_\alpha=2(1/2-\alpha)^2$; its contribution to expected risk is
$O((1+k)e^{-c_\alpha n})$ by the same coefficient bound.
Thus the comparator's expected risk vanishes, and
\eqref{eq:theory-validation-bound} proves the claim. The selector itself
never observes the corruption indicators.
\end{proof}
\par}

\subsection{Filtering an unknown corrupted set}
\label{app:theory-filtering}

{\color{black}
Keep $T=I_2$ and $\theta^*=e_1$. Now hold $n,n',D,\lambda$ fixed and let
$\delta\to\infty$ in~\eqref{eq:theory-biased-teacher}.
Every contaminated subset's target loss eventually diverges as $\delta$ grows,
while clean fits are unaffected. Selection therefore rejects all corruption
whenever a clean subset exists.
\par}

\begin{theorem}[Filtering]
\label{thm:theory-filtering}
Fix finite $n,n',D$ with $n'\ge1$ and $D+2>n+n'$, and fix $\lambda>0$.
Let $0<\alpha<1/2$ and $k=\lfloor\alpha n\rfloor\ge1$. Then
\begin{equation}
\label{eq:theory-filtering-result}
\begin{aligned}
    \lim_{\textcolor{black}{\delta}\to\infty}\mathbb{P}\{\text{the selected subset is entirely clean}\}
       &=\mathbb{P}\{\operatorname{Bin}(n,1/2)\ge k\}\\
       &\ge1-e^{-c_\alpha n},
       \qquad c_\alpha=2(1/2-\alpha)^2.
\end{aligned}
\end{equation}
Conditional on the existence of a clean $k$-subset, the filtering probability
tends to one.
\end{theorem}

{\color{black}
\begin{proof}
Fix the source and target inputs across values of $\delta$. With
$r_i=\operatorname{sign}(g_{i2})\mathbf{1}\{h_{i1}>0\}$, linearity gives
$\beta_S(\delta)=\beta_S(0)+\delta u_S$, where $u_S=X_S^\dagger r_S$.
For a contaminated subset, full row rank implies $u_S\ne0$.
The independent Gaussian target inputs have positive-definite covariance, so
$\|Zu_S\|^2>0$ almost surely, even for $n'=1$. Its target loss therefore grows
quadratically with $\delta$; a clean subset's loss stays constant.
There are finitely many subsets, so every minimizer is eventually clean
exactly when at least $k$ clean examples exist. Bounded convergence and
Hoeffding's inequality for their $\operatorname{Bin}(n,1/2)$ count give the result.
\end{proof}
\par}

\subsection{Covariate shift with correct labels}
\label{app:theory-covariate}

{\color{black}
Set $\delta=0$, so $\widehat f=f^*$ and both source and target labels are correct.
Selecting data can still help by changing the learner's inductive bias.
Set $T=T_n=\operatorname{diag}(1,1/n)$, so the target variance in the second
coordinate shrinks as the source pool grows. Selecting examples with $g_1g_2>0$
correlates the signal features; the coefficients below make the limiting
ridge error point entirely along that second coordinate, giving vanishing
target risk. Set $\lambda=k$, so with fixed $n'$, fine-tuning makes a vanishing
correction to the signal coefficients as the source pool grows.
\par}

\begin{equation}
\label{eq:theory-covariate-construction}
    \rho=\mathbb{E}|g_1g_2|=\frac{2}{\pi},\qquad
    L=\sqrt{\rho^2+4},\qquad
    \theta^*=\frac{1}{L}\begin{pmatrix}\rho\\2\end{pmatrix}.
\end{equation}

{\color{black}
In particular, $\theta^*$ is a unit vector independent of $\alpha$ and $n$.
\par}

\begin{theorem}[Covariate shift]
{\color{black}
\label{thm:theory-covariate}
Fix $0<\alpha<1/2$ and $n'\ge4$. Set $k=\lfloor\alpha n\rfloor$,
$\lambda=k$, and $T=T_n=\operatorname{diag}(1,1/n)$, and use the parameter vector
in~\eqref{eq:theory-covariate-construction}. Then
\begin{equation}
\label{eq:theory-covariate-result}
\begin{aligned}
    \lim_{n\to\infty}\mathcal{R}_{\mathrm{DW}}(n)&=0,\\
    \lim_{n\to\infty}\min\{\mathcal{R}_{\mathrm{full+FT}}(n),
        \mathcal{R}_{\mathrm{random+FT}}(n)\}&\ge\frac{\alpha^2}{25}>0.
\end{aligned}
\end{equation}
\par}
For the same $\theta^*$ and procedures with $T=I_2$,
{\color{black}
\begin{equation}
\label{eq:theory-isotropic-result}
    \liminf_{n\to\infty}\mathcal{R}_{\mathrm{DW}}(n)
    \ge\left(\frac{\alpha}{1+\alpha}\right)^2
    =\lim_{n\to\infty}\mathcal{R}_{\mathrm{full+FT}}(n).
\end{equation}
\par}
\end{theorem}

{\color{black}
\begin{proof}
The source fits converge to $\theta^*/(1+\alpha)$ and $\theta^*/2$ for full-data
and random-subset training, respectively. Couple the target signal matrices as
$V_n=UT_n^{1/2}$ with a fixed standard Gaussian $n'\times2$ matrix $U$.
Since $T_n\preceq I_2$, we have $\|V_n^\top V_n\|_{\mathrm{op}}
\le\|U\|_{\mathrm{op}}^2$. The fine-tuning error map in
\eqref{eq:theory-ft-reduction} therefore tends to $I_2$ as $\lambda=k\to\infty$.
Their coefficient errors have norm at most one before and after fine-tuning;
with $T_n=\operatorname{diag}(1,1/n)$, bounded convergence gives
\begin{equation}
\label{eq:theory-covariate-baselines}
    \mathcal{R}_{\mathrm{full+FT}}(n)
       \longrightarrow\left(\frac{\alpha}{1+\alpha}\right)^2
           \frac{\rho^2}{L^2},
    \qquad
    \mathcal{R}_{\mathrm{random+FT}}(n)
       \longrightarrow\frac{\rho^2}{4L^2}.
\end{equation}
The full-data limit is smaller and equals
$\alpha^2/[(1+\alpha)^2(1+\pi^2)]\ge\alpha^2/25$.

For a comparator, take the first $k$ examples with $g_1g_2>0$, falling back to
the first $k$ examples if necessary. The fallback probability is exponentially
small, and the selected signal covariance and limiting error are
\begin{equation}
    B=\begin{pmatrix}1&\rho\\\rho&1\end{pmatrix},
    \qquad
    \theta_{S_0}-\theta^*\longrightarrow-(B+I_2)^{-1}\theta^*
       =-\frac{e_2}{L}.
\end{equation}
Here $B=\mathbb{E}[gg^\top\mid g_1g_2>0]$, and
$(B+I_2)e_2=L\theta^*$ explains the chosen coefficients.
Write $e_n=\theta_{S_0}-\theta^*$. Its norm is at most one even on fallback,
and $e_{n,1}\to0$, so bounded convergence gives
\begin{equation}
    \mathbb{E}R_{T_n}(\theta_{S_0})
    =\mathbb{E}\bigl[e_{n,1}^2+n^{-1}e_{n,2}^2\bigr]
    \le\mathbb{E}e_{n,1}^2+n^{-1}\longrightarrow0.
\end{equation}
Applying \eqref{eq:theory-validation-bound} at each $n$, with the same constant
$C_{n'}$, proves $\mathcal{R}_{\mathrm{DW}}(n)\to0$.

Finally, take $T=I_2$. Write $A_S=G_S^\top G_S/n$,
$A_{\mathrm{all}}=G^\top G/n$, and $a_n=k/n$.
Since $A_S\preceq A_{\mathrm{all}}\to I_2$, every subset's ridge error satisfies
\begin{equation}
    \|\theta_S-\theta^*\|^2
    =a_n^2\|(A_S+a_n I_2)^{-1}\theta^*\|^2
    \ge\left(\frac{a_n}{\lambda_{\max}(A_{\mathrm{all}})+a_n}\right)^2
    \longrightarrow\left(\frac{\alpha}{1+\alpha}\right)^2.
\end{equation}
This uniform bound also holds for target-selected subsets. Taking the
finite-$D$ limit and applying Fatou's lemma as $n\to\infty$ proves the lower
bound. For full-data training, the coefficient error tends to
$-\alpha\theta^*/(1+\alpha)$, and fine-tuning's error map tends to $I_2$.
Since $\|\theta^*\|=1$, bounded convergence gives the matching risk limit
$[\alpha/(1+\alpha)]^2$.
\end{proof}

\paragraph{\normalcolor Connection to the informal statements.}
For each sufficiently large $n$, choose $d(n)>n+n'$ so that the
finite-dimensional expected risk of dataweighting is at most
$\mathcal{R}_{\mathrm{DW}}(n)+1/n$, and each baseline's expected risk is within
$1/n$ of its $D\to\infty$ limit. Such a choice follows from the definition of
$\limsup$ and the convergence of the baseline risks proved above.
The limits as $n\to\infty$ then give the statements in the main text.
For bias correction, taking $\lambda=1$ gives
$c_{n'}=1/[2\pi(1+n')^2]$.

For covariate shift, an orthogonal rotation of the two signal coordinates
maps the unit vector $\theta^*$ to $e_1$, giving $\widehat f=f^*=x_1$.
The source signal covariance remains $I_2$, while the target signal
covariance rotates with the coordinates. Minimum-norm fitting, gradient-flow
fine-tuning, subset selection, and population risks are preserved by this
change of basis.
\par}

\par}

\end{document}